%% file: main.tex
\documentclass{amsart}

\usepackage{amsmath}
\allowdisplaybreaks

        {\hspace*{\fill}$\qed$\par}

\input{macros.tex}

\usetikzlibrary{calc}

\begin{document}
\title{Rk-means: Fast Clustering for Relational Data}

\author{Ryan Curtin}
\address{relational\underline{AI}}

\author{Ben Moseley}
\address{Carnegie Mellon University}

\author{Hung Q. Ngo}
\address{relational\underline{AI}}

\author{XuanLong Nguyen}
\address{University of Michigan}

\author{Dan Olteanu}
\address{Oxford University}

\author{Maximilian Schleich}
\address{Oxford University}

\begin{abstract}
Conventional machine learning algorithms cannot be applied until a data matrix
is available to process. When the data matrix needs to be obtained
from a relational database via a feature extraction query, the computation
cost can be prohibitive, as the data matrix may be (much) larger than the
total input relation size.  This paper introduces
Rk-means, or relational $k$-means algorithm, for clustering relational data
tuples without having to access the full data matrix. As such, we
avoid having to run the expensive feature extraction query and storing its output.  
Our algorithm leverages the underlying structures in relational data. It involves
construction of a small {\it grid coreset} of the data matrix for subsequent
cluster construction.  This gives a constant approximation for 
the $k$-means objective, while having asymptotic runtime improvements over standard 
approaches of first running the database query and then clustering. 
Empirical results show orders-of-magnitude speedup, and Rk-means can run 
faster on the database than even just computing the data matrix. 
\end{abstract}

\maketitle

\section{Introduction}
\label{sec:intro}

Clustering is an ubiquitous technique for exploratory data analysis,
whether applied to small samples or industrial-scale data.
In the latter setting, two steps are typically performed:
{\it (1) data preparation}, or extract-transform-load (ETL) operations, and {\it
  (2) clustering the extracted data}---often with a technique such
  as the popular $k$-means algorithm~\cite{cady2017data,WuKQGYMMNLYZSHS08}.
In this setting, data typically reside in a relational database, requiring
a {\it feature extraction query} (FEQ) to be performed, \emph{joining} involved relations together
to form the data matrix: each row corresponds to a data tuple and
each column a feature. Then, the data matrix is used
as input to a clustering algorithm.
Such data matrices can be expensive to compute, and may take up space
asymptotically larger than the database itself, which is made of relational tables.
Moreover, the join computation time may exceed the time it takes to 
obtain clusters.  It is not uncommon that the exploratory trip into the dataset may be stopped 
right at the gate.

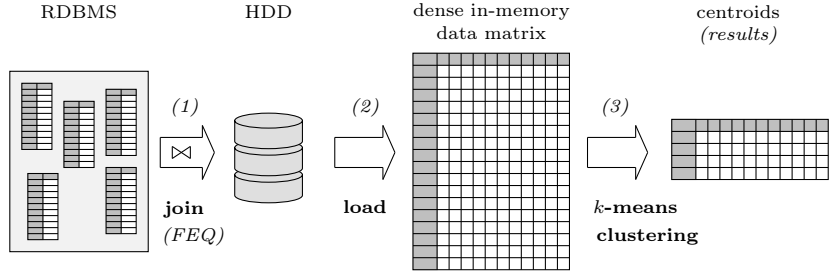
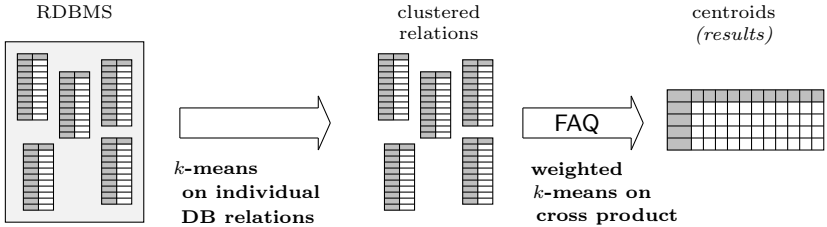
\begin{figure}[t]
\subfigure[Typical $k$-means data science workflow.  Alternate representations
can be used for the data in step \textit{(2)} for greater computational
efficiency (e.g., streaming); and, approximation strategies are known for step
\textit{(3)}.  However, the dataset often comes from an underlying database
system, and in this case the expensive FEQ join \textit{(1)} is
unavoidable.\vspace*{-0.7em}]{
\begin{tikzpicture}[scale=0.8]
\input{regular-workflow.tex}
\end{tikzpicture}
}
\subfigure[The Rk-means data science workflow.  We avoid ever computing the
expensive FEQ by instead clustering each underlying relation (steps
1 and 2, Section \ref{sec:structures}); we then use \faqs\ for efficient
weighted $k$-means of the cross-product of those relations (steps 3
and 4, Section \ref{sec:structures}).  This gives significant empirical and
theoretical accelerations, and bounded approximation---{\it without ever
computing the full data matrix.}]{
\begin{tikzpicture}[scale=0.8]
\input{rml-workflow.tex}
\end{tikzpicture}
}
\label{fig:workflows}
\caption{Conventional $k$-means and Rk-means.}
\end{figure}

As an example, consider
a retailer database consisting of three tables: {\small \tt product},
which contains data about $p$ products, {\small \tt store},
which contains data about $s$ stores, and {\small \tt
transaction}, which contains the number of transactions for each
(product, store) combination on a given day.
The table {\small \tt product} contains
information about each of the $p$ products, {\small \tt stores} contains
information about each of the $s$ stores, and {\small \tt transactions} contains
the (nonzero) number of transactions for each (product, store) combination on a
given day. 
A practitioner may want to cluster each (product, store)
combination as part of an analysis to determine items with related
sales patterns across different stores for a given week.
To do this, she constructs a data
matrix containing all (product, store) combinations (including those with
zero sales) for a given week, and additional attributes for each product and
store. This can be achieved, for instance, by the following feature extraction
query, given in SQLite syntax:

\begin{center}
\begin{verbatim}
SELECT P.id AS i, S.id AS s, P.type AS t, P.price AS p,
           S.yelp_rating AS y, sum(ifnull(T.count, 0)) AS c
       FROM product P, store S LEFT JOIN transactions T
       ON T.product_id == P.id AND T.store_id == S.id
           AND T.date BETWEEN '2019-05-13' AND '2019-05-20'
       GROUP BY P.id, S.id;
\end{verbatim}
\end{center}

The result of this query is of size $\Theta(ps)$.  But the {\small \tt
  transaction} table can be significantly smaller than this, since many stores
may have zero sales of a particular product in a given week.  Thus, the size of
the data matrix can be asymptotically greater than the total input relations'
sizes.
Real-world FEQs possess a similar explosion in both space and time complexity,
only at a much larger scale, since they generally involve many more aggregations and
tables.  
In Section~\ref{sec:experiments}, we present a real dataset 
from a large US retailer. The database has 6 tables of total size 1.5GB. The
FEQ result, however, takes up 18GB, and constructing it
takes longer than running a learning algorithm on it. 

Stripping away the language of databases, a fundamental challenge is how to learn
about the joint distribution of a data population given only
marginal samples revealed by relational tables. This is possible when
the objective function of an underlying model admits some factorization
structure similar to conditional independence in graphical models~\cite{MR2778120}.
This insight was exploited recently by database theorists to 
devise algorithms evaluating a generic class of relational queries 
called {\it functional aggregate queries}, or \faqs~\cite{DBLP:conf/pods/KhamisNR16}.
The ability to answer $\faqs$ quickly is a building block for a new class
of efficient algorithms for training supervised learning models over relational data,
{\em without} having to materialize (i.e. compute) the entire data 
matrix~\cite{ANNOS:PODS:2018, ANNOS:DEEM:2018, Schleich:2016:LLR:2882903.2882939}.

The goal of this paper is to devise a method for fast clustering of relational
data, without having to materialize the full data matrix. The challenge
of unsupervised learning tasks in general and the $k$-means algorithm in particular
is that the learning objective is not decomposable across marginal samples
in relational tables. To enable fast relational computation, we utilize the
idea of constructing a {\it grid coreset}---a small set of points that provide a
good summarization of the original (and unmaterialized) data tuples,
based on which a provably good clustering can be obtained.

The resulting algorithm, which we call Rk-means, has several remarkable
properties. First, Rk-means has a provable constant approximation guarantee
relative to the $k$-means objective, despite the fact that the algorithm
does not require access to the full data matrix. Our approximation analysis
is established via a connection of Rk-means to the theory of optimal transport~\cite{MR2459454}.
Second, Rk-means is enhanced by leveraging structures prevalent in relational data: 
categorical variables, functional dependencies, and the topology of feature extraction
queries. These structures lead to exponential reduction in
coreset size without incurring loss in the coreset's approximation error.
We show that Rk-means is provably more efficient both in time and space
complexity when comparing against the application of the vanilla $k$-means to the
full data matrix. 
Finally, experimental results show significant speedups with little loss in the
$k$-means objective.  We observe orders-of-magnitude improvement in the
running time compared to traditional methods. Rk-means is able to operate
when other approaches would run out of memory, enabling clustering on
truly massive datasets.

\section{Background and related work}
\label{sec:related}

\subsection{Background on Database Queries and FAQs}
\label{sec:database_queries}

Recent advancements in the database community have produced new classes of
query plans and join algorithms~\cite{DBLP:journals/sigmod/Khamis0R17,
DBLP:conf/pods/KhamisNR16,DBLP:conf/pods/000118,DBLP:journals/sigmod/OlteanuS16}
for the
efficient evaluation of general database queries. These general algorithms hinge on the
expression of a database query
as a {\em functional aggregate query}, or $\faq$~\cite{DBLP:conf/pods/KhamisNR16}.

Loosely speaking, an $\faq$ is a collection of {\it aggregations} (be they sum,
max, min, etc.) over a number of functions known as {\it factors}\footnote{A
full formal definition of {\faq}s can be found in ~\cite{DBLP:conf/pods/KhamisNR16}, but
is not required for our work here so we omit it.}, in the same sense as that used in
graphical models. In particular, if there was only one aggregation 
(such as {\sf sum}), then an $\faq$ is just a sum-product form typically used to compute 
the partition function. An $\faq$ is more general as it can involve many marginalization
operators, one for each variable, and they can interleave in arbitrary way.
Every relational database query can be
expressed in this way. Consider the example query of Section~\ref{sec:intro}:
for this, the task of the database query evaluator is to compute {\small \tt
max(transactions.count)} for every tuple $(i, s, t, p, y)$ that exists in the
output.  We can express this as a function:
\begin{align}
\phi(i, s, t, p, y) = \max_c \max_i \max_s \psi_P(i, t, p) \psi_T(i, s, c) \psi_S(s, y).
\label{eqn:example_faq}
\end{align}
In this we have three {\it factors} $\psi_P(\cdot)$, $\psi_T(\cdot)$, and
$\psi_S(\cdot)$, which correspond to the {\small \tt product}, {\small \tt
transactions}, and {\small \tt store} tables, respectively.  We define
$\psi_P(i, t, p) = 1$ if the tuple $(i, t, p)$ exists in the {\small \tt
product} table and $0$ otherwise; we define $\psi_S(s, y)$ similarly.  We define
$\psi_T(i, s, c) = c$ if the tuple $(i, s, c)$ exists in the {\small \tt
transactions} table and $0$ otherwise.  Thus, given any tuple $(i, s, t, p, y)$,
we can compute {\small \tt max(transactions.count)}$ = \phi(i, s, t, p, y)$.

In order to efficiently solve an $\faq$\ (of which Equation~\eqref{eqn:example_faq}
is but one example), the $\InsideOut$\ algorithm of~\cite{DBLP:conf/pods/KhamisNR16} may be 
used; $\InsideOut$ is a variable elimination algorithm, inspired from variable elimination
in graphical model, with several new twists. One twist is to adapt 
worst-case optimal join algorithms~\cite{} to speed up computations by exploiting sparsity
in the data. Another twist is that the algorithm has to carefully pick a variable order
to minimize the runtime, while at the same time respect the correctness and semantic of the
query. Unlike in the case of computing a sum-product where the summation opeartors are
commutative, in a $\faq$ the operators may not be commutative.

To characterize the runtime of this algorithm, we must first observe that each
database query and thus \faq\ corresponds to a hypergraph $\mathcal{H} = \{
\mathcal{V}, \mathcal{E} \}$.  The vertices $\mathcal{V}$ of this hypergraph
correspond to the variables of the \faq\ expression; in our example, we have
$\mathcal{V} = \{ i, s, t, p, y, c \}$.  The hyperedges $\mathcal{E}$, then,
correspond to each factor $\psi_P(\cdot)$, $\psi_T(\cdot)$, and
$\psi_S(\cdot)$---which in turn correspond to the tables in the database.  This
hypergraph $\mathcal{H}$ is shown in Figure~\ref{fig:example_hypergraph}.

\begin{figure}[h!]
\begin{center}
\begin{tikzpicture}
  \node (t) at (-2.5, 0.75) { };
  \node (p) at (0.5, 0.75) { };
  \node (i) at (-1.0, 0.0) { };
  \node (c) at (1.0, -0.4) { };
  \node (s) at (3.5, 0.0) { };
  \node (y) at (3.5, 0.75) { };

  \begin{scope}[fill opacity=0.8]
    \filldraw[fill=yellow!70] ($(t) + (0.0, 0.3)$)
        to [out=0,in=180] ($(p) + (0.0, 0.3)$)
        to [out=0,in=0] ($(i) + (0.0, -0.5)$)
        to [out=180,in=180] ($(t) + (0.0, 0.3)$);

    \filldraw[fill=red!70] ($(s) + (0.0, -0.35)$)
        to [out=180,in=180] ($(y) + (0.0, 0.35)$)
        to [out=0,in=0] ($(s) + (0.0, -0.35)$);

    \filldraw[fill=green!70] ($(c) + (0.0, -0.2)$)
        to [out=180,in=190] ($(i) + (-0.2, 0.15)$)
        to [out=10,in=170] ($(s) + (0.2, 0.15)$)
        to [out=350,in=0] ($(c) + (0.0, -0.2)$);
  \end{scope}

  \node at (-1.0, 0.6) { $\psi_P$ };
  \node at (1.0, 0.0) { $\psi_T$ };
  \node at (3.5, 0.4) { $\psi_S$ };

  \fill (t) circle (0.1) node [right] {$t$};
  \fill (p) circle (0.1) node [left] {$p$};
  \fill (i) circle (0.1) node [right] {$i$};
  \fill (c) circle (0.1) node [right] {$c$};
  \fill (s) circle (0.1) node [right] {$s$};
  \fill (y) circle (0.1) node [right] {$y$};
\end{tikzpicture}
\end{center}
\label{fig:example_hypergraph}
\caption{Example hypergraph $\mathcal{H}$ for the example query and \faq\ in
Equation~\ref{eqn:example_faq}.}
\end{figure}
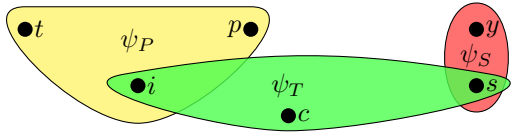

Roughly, \InsideOut\ proceeds by first selecting a variable ordering
$\sigma$, reordering the \faq\ accordingly, and then solving the inner
subproblems repeatedly, in much the same way variable elimination works for
inference in graphcal models~\cite{MR2778120}.
The runtime of \InsideOut\ is dependent on a notion of width
of $\mathcal{H}$ called \faq-width, or $\operatorname{faqw}(\cdot)$.  Fully
describing this width is beyond the scope of this paper and we encourage readers
to refer to~\cite{DBLP:conf/pods/KhamisNR16} for full details.
The \faq-width is a generalized version of {\em fractional hypertree width}
of~\cite{DBLP:journals/talg/GroheM14} (denoted by $\fhtw$).
When the $\faq$ query does not have free variables, $\faqw = \fhtw$.
Given some $\faq$ with hypergraph $\mathcal{H}$, via Section
4.3.4 of~\cite{DBLP:journals/sigmod/Khamis0R17}, $\InsideOut$ runs in time
$\tilde{O}(N^{\operatorname{faqw}_{\mathcal{H}}(\sigma)} + Z)$,
where we assume that the support of each factor\footnote{Or in our
case, the number of tuples in the table corresponding to that factor.} is no
more than $O(N)$, and $Z$ is the number of tuples in the output.  As an example,
the hypergraph of Figure~\ref{fig:example_hypergraph} has
$\operatorname{faqw}_{\mathcal{H}}(\sigma) = 1$.  Overall, \InsideOut\ gives us
the most efficient known way to evaluate problems that can be formulated as
\faqs.

\subsection{Coresets for clustering} 

From early work on $k$-means algorithm~\cite{Lloyd-82}, ideas emerged for acceleration via
coresets~\cite{Har-Peled-2004,Bachem-2017}.  Coresets have become the cornerstone of
modern streaming algorithms \cite{GuhaMMMO03,BravermanFLSY17}, massively parallel
(MPC) algorithms \cite{EneIM11,BahmaniMVKV12}, and are used to speed up
sequential methods \cite{MeyersonOP04,SohlerW18}.

Unfortunately, existing algorithms for coreset construction do not readily lend
themselves to the relational setting; there are several hurdles.  First of all,
coresets are formed by constructing a set $S$ of data points (tuples) that represent the
entire data set $X$ well. Typically, $S$ is a \emph{weighted} representation of
the data, where each point in the universe contributes one unit of weight to its
closest point in $X$ \cite{Peled,EneIM11,BalcanEL13}.  In our relational setting, 
$X$ can only be formed by computing the FEQ, but our goal is to avoid materializing $X$.

A common challenge for adapting existing coreset constructions given our
goal is that most methods construct $S$ in phases by determining the farthest 
points from $S$ \cite{Thorup01,Arthur-2007,Peled}.  This is difficult without
$X$ fully materialized.
%
Another difficulty is that, even if the points in $S$ are given, weighting the points in $S$ 
is an open problem for relational algorithms \cite{PODSInequalities}. Without $X$
materialized, again, the points and their attributes are stored across several tables.
Fixing a point $x \in S$ and finding the number of closest points in 
(unmaterialized) $X$ is non-trivial.  No method, either deterministic or
stochastic (e.g., sampling),
is known that runs in time asymptotically faster than computing/materializing $X$.  Our method avoids this 
by constructing a \emph{grid coreset} $S$ which can be
decomposed over the tables in such a way that computing the weights of the
points is a straightforward task.


\subsection{Other Related Work}

Our work draws inspiration from three lines existing work and ideas: coresets for
clustering (discussed above), relational algorithms, and optimal transport.  
Some previous work has focused on the connection to databases; database and disk
hardware optimizations have been considered to improve clustering relational
data~\cite{Ordonez06, OrdonezO04}.  Recent advances include
the work of~\cite{ANNOS:PODS:2018,Schleich:2016:LLR:2882903.2882939}.
$k$-means has also been connected to optimal transport, which goes back to at
least~\cite{Pollard-82} (see also~\cite{Graf-Luschgy-00}). Recently this
connection has received increased interest in the statistics and machine
learning communities, resulting in fresh new clustering
techniques~\cite{delBarrio-2017,Ho-2017,Ye-2017}.
To our knowledge, these related lines of work have not been explored
together. Motivated by clustering relational data, our attempt at solving a 
clustering problem formulated as optimal transport in the marginal 
(projected) spaces to scalably perform $k$-means clustering
appears to be the first in the literature.

Finally, it is worth noting that despite its popularity, the basic $k$-means technique is not always 
a preferred choice in clustering categorical or high-dimensional data. One may
either adopt other clustering techniques~\cite{Hartigan-75,Kaufman-Roussew-90,Everitt-etal-11}, 
or modify the basic $k$-means method,
e.g., by suitably placing weights on different features of mixed data types
and replacing metric $\ell_2$ by $\ell_p$
~\cite{Huang-1998}, or incorporating a regularizer to combat high
dimensionality~\cite{Witten-Tibshirani-10,Sun-Wang-2012}. As we shall see, the relational techniques and associated
theory that we introduce for the basic $k$-means extend easily to such improvements.




\section{Rk-means, coresets and optimal transport}
\label{sec:algorithm}

Although the Rk-means algorithm is motivated by application to
relational databases, its basic idea is also of independent interest
and can be easily described without the database language.


First we define the \emph{weighted} $k$-means problem, which Rk-means solves 
(weights are also handy in combining mixed data types~\cite{Huang-1998}). Let
$\Xset$ be a set of points in $\mathbb{R}^d$, and $\bm Y$ be a non-empty set of
points in the same space.  Let $d(\bm x,\bm Y) := \min_{\bm y \in \bm Y}
\norm{\bm x-\bm y}$ denote the minimum distance from $\bm x$ to an
element in $\bm Y$. In some cases, the $\ell_2$ norm $\|\cdot\|$ may be replaced 
by the $\ell_p$ norm $\|\cdot\|_p$ for some $p\geq 1$.
A {\em weighted $k$-means instance} is a pair $(\bm X, w)$, where
$\bm X$ is a set of points in $\R^d$ and $w : \bm X \to \R^+$ is a weight
function.
Without loss of generality, assume $\sum_{\bm x\in \Xset} w(\bm x) = 1$.
The task is to find a set
$\bm C=\{\bm\mu_1,\cdots,\bm\mu_k\}$ of $k$ centroids to minimize the objective
$L(\Xset,\bm C, w) = \sum_{x\in \Xset} w(\bm x) d(\bm x,\bm C)^2.$

\setlength{\textfloatsep}{7pt}
\begin{figure}
\begin{algorithm}[H]
\begin{algorithmic}[1]
   \STATE {\bf Input:} query $Q$, number of clusters $k$
   \STATE {\bf Input:} $[d] = S_1 \cup \cdots \cup S_m$, $\kappa \geq 2$
   \STATE {\bf Output:} centroids $\bm C \in \mathcal{R}^{k \times d}$

   \medskip
   \FOR {$j = 1$ to $m$}
     \STATE $\bm X_j \gets \{ \bm x_{S_j} \suchthat \bm x \in \bm X\}$
     \STATE $w_j \gets $ weight function defined in~\eqref{eqn:wj}
      \STATE $\bm C_j \gets \wkmeans_1(\bm X_j, w_j, \kappa)$ \COMMENT{approx. ratio $\alpha$} \label{line:alpha}
   \ENDFOR

   \medskip

   \STATE $\bm G \gets \bm C_1 \times \ldots \times \bm C_m$ \COMMENT{the grid coreset}
   \STATE $\wgrid \gets $ weight function defined in~\eqref{eqn:wg}
   \STATE $\bm C \gets \wkmeans_2(\bm G, \wgrid, k)$ \COMMENT{approx. ratio $\gamma$} \label{line:gamma}
\end{algorithmic}
\caption{{\bf Rk-means}: $k$-means via grid-coreset}
\label{alg:kmeans}
\end{algorithm}
\end{figure}

That is, we want to solve the problem $\opt(\bm X, w) :=
\min_{\bm C} L(\Xset, \bm C, w)$.  With Rk-means, we will do this by projecting
$\Xset$ onto different sets of coordinates, and clustering each projection
individually.  To this end, let $[d] = S_1 \cup \cdots \cup S_m$ denote an
arbitrary {\em partition} of the dimensions $[d]$ into non-empty subsets.
For every $\bm x \in \R^d$, and $j \in [m]$, let $\bm x_{S_j}$ denote the
projection of $\bm x$ onto the coordinates in $S_j$. Define the projection set
$\bm X_j$ and corresponding weight function $w_j : \R^{S_j} \to \R$  by

\begin{align}
    \bm X_j &:= \left\{ \bm x_{S_j} \suchthat \bm x \in \bm X \right\}, \\
    w_j(\bm z) & := \!\!\!\!\sum_{ \bm x \in \bm X \ : \ \bm x_{S_j} = \bm z}
\!\!w(\bm x).
\label{eqn:wj}
\end{align}

In words, the $w_j$ are the {\em marginal measures} of $w$ on the subspace of
coordinates $S_j$.  With these notations established, Algorithm \ref{alg:kmeans}
presents the high-level description of our algorithm, Rk-means.

For each $j \in [m]$, in line 7 we perform $k$-means to obtain $\kappa$ individual
clusters on each subspace $S_j$ for some $\kappa \geq 2$.
These are solved using some weighted $k$-means algorithm denoted
by $\wkmeans_1$ with approximation ratio $\alpha$.
Then, using the results of these clusterings, we assemble a cross-product weighted grid
$\bm G$ of centroids, and then perform
$k$-means clustering on these using the algorithm denoted $\wkmeans_2$ to reduce
down to the desired result of $k$ centroids.
Typically, take $\kappa = O(k)$.

Let $\bm X := \biguplus_{\bm g \in \bm G} \bm X_{\bm g}$
denote a partition of $\bm X$ into $|\bm G|$ parts, where $\bm X_{\bm g}$ denote the set
of points in $\bm X$ closer to $\bm g$ than other grid points in $\bm G$ (breaking ties
arbitrarily). Then, the weight function for line 11 is $\wgrid : \bm G \to \R^+$
is defined as

\begin{align}
    \wgrid(\bm g) := \sum_{\bm x \in \bm X_{\bm g}} w(\bm x).
    \label{eqn:wg}
\end{align}


\paragraph*{Weighted $\kmeans$ and optimal transport}
We will analyze the Rk-means algorithm in the language of optimal transport. The
connection of $k$-means in general, and of our algorithm to optimal transport in
particular, provides another interesting insight into our algorithm.

The {\em optimal transport distance} characterizes the
distance between two probability measures, by measuring the optimal cost of transporting
mass from one to another~\cite{MR2459454}. Although this is defined more generally for any
two probability measures in abstract spaces, for our purpose it is convenient to consider
two discrete probability measures $P$ and $P'$ on $\mathbb{R}^d$.

Let $\bm Z$ and $\bm Z'$ be two finite point sets in $\R^d$.
Let $\delta$ denote the Dirac measure.
Let $P := \sum_{\bm z \in \bm Z} p(\bm z) \delta_{\bm z}$ and
$P' := \sum_{\bm z' \in \bm Z'} p'(\bm z') \delta_{\bm z'}$
be two measures with supports $\bm Z$ and $\bm Z'$, respectively.
The mass transportation plan can be formalized by a {\em coupling}: a joint distribution
$\bm Q = (q({\bm z,\bm z'}))_{(\bm z,\bm z') \in \bm Z\times \bm Z
}$, where the marginal constraints
$\sum_{\bm z\in \bm Z} q_{\bm z,\bm z'} = p'(\bm z')$ and
$\sum_{\bm z'\in \bm Z'} q_{\bm z,\bm z'} = p(\bm z)$ hold.

\bdefn
For any $p\geq 1$, the {\em Wasserstein distance} of order $p$ is defined by the 
minimization of $\bm Q$ over
all possible couplings:
    $\Wassersteinp(P,P') = \min_{\bm Q} ( \{\sum_{\bm z, \bm z'} q(\bm z,\bm z')
    \norm{\bm z - \bm z'}_p^p ) \}^{1/p}.$
\edefn

Let $\Pn = \sum_{\bm x\in \Xset}\! w(\bm x)\delta_{\bm x}$ be the discrete measure associated with
the input instance of our weighted $k$-means problem; then,
this can be expressed precisely as an optimal transport problem:
 $M^* = \argmin \; \Wasserstein^2(M, \Pn),$
where the optimization is over the space of discrete measures $M$ that have $k$
support points (the set $\bm C$ of $k$ centroids).
Note that $\opt(\bm X, w) = \Wasserstein^2(M^*, \Pn)$. Replacing $\ell_2$ by say $\ell_1$, 
we obtain the $k$-median problem, for which the objective becomes $W_1(M^*,\Pn)$.

\paragraph*{Approximation Analysis.}
We next analyze the approximation ratio of our Rk-means algorithm working with
the $\Wasserstein^2$ 
objective, provided that $\wkmeans_1$ has approximation ratio $\alpha$ and
$\wkmeans_2$ has approximation ratio $\gamma$.\footnote{The best known
approximation ratio is $6.357$ for data in Euclidean space
\cite{AhmadianNSW17}.} 
The reason we might want to invoke different algorithms to solve these sub-problems is
because, as we shall show in the next section, we may want to exploit the (relational)
structures of the FEQ to construct a ``nice'' partition $S_1 \cup \cdots \cup S_m$.
We show that the overall approximation ratio of Rk-means is
$(\sqrt \alpha+\sqrt \gamma+\sqrt{\alpha\gamma})^2$.  In many common cases, the
relational database has structure that allows $\alpha = 1$, yielding an overall
approximation ratio of $(1 + 2\sqrt\gamma)^2$.

For our analysis it is useful to understand Algorithm~\ref{alg:kmeans} in the
language of optimal transport.  For any finite point set $\bm Y \subset \R^d$
and a measure $M = \sum_{\bm y \in \bm Y} p(\bm y) \delta_{\bm y}$ with support
$\bm Y$, define the marginal measures $M_j$ on coordinates $S_j$ {\em induced}
by $M$ in the natural way, i.e.  $M_j := \sum_{\bm z \in \bm Z} p_j(\bm z)
\delta_{\bm z}$ where $p_j$ is defined analogous to $w_j$ in~\eqref{eqn:wj}.
Under this notation, $\Pn$ induces the marginal measures $\Pn_j:=\sum_{\bm z\in
\Xset_j} w_j(\bm z) \delta_{\bm z}$.  Then, Algorithm~\ref{alg:kmeans} can be
described by the following steps:

\begin{itemize}
    \item [(1)] For each $j \in [m]$, pick $M_j$ to be the ($\alpha$-approximate) minimizer of
        $\Wasserstein^2(M_j, \Pn_j)$, where $\bm C_j = \textrm{supp}(M_j)$ is the support of $M_j$ 
and $|C_j| = \kappa$ (line 7).

\item [(2)] Collect the $\kappa^d$ grid points $\bm G$
and let probability measure $Q$ be the one with support in $\bm G$ such that $Q$
minimizes $\Wasserstein^2(Q, \Pn)$. (We solve this problem exactly!)

\item [(3)] Finally, return $P$ which is the measure with exactly $k$ support points
    in $\mathbb{R}^d$ that ($\gamma$-approximately)
        minimizes $\Wasserstein^2(P,Q)$ (line 11). 
\end{itemize}

This is precisely the solution
obtained by Algorithm~\ref{alg:kmeans}. We present next some useful facts.

\blmm
For any discrete measure $M$ on $\R^d$,
$\Wasserstein^2(M,\Pn) \geq \sum_{j=1}^{m} \Wasserstein^2(M_j, \Pn_j).$
\label{prop:relateOpt}
\elmm
\begin{proof}
A valid coupling of two measures induces valid marginal couplings of marginal measures.
\end{proof}

\bprop
\label{W-bounds}
The following hold:
\begin{itemize}
    \item[(a)] If $\kappa \geq |\textrm{supp}(M_j^*)|\;\forall\;j \in [m]$, then $\Wasserstein(\Pn,P) \leq (\sqrt \gamma+ \sqrt \alpha + \sqrt{\gamma \alpha}) \Wasserstein(\Pn, M^*)$. 
    \item[(b)] For any $\kappa \geq 1$, there exists a distribution $\Pn$ such that 
\begin{align}
\frac{\Wasserstein(\Pn,P)}{\Wasserstein(\Pn,M^*)}
&\geq \sqrt{1- e^{-m/(2\kappa)}} \frac{\sqrt{3}k^{3/(2m)}}{2\kappa m^{1/2}}.
\end{align}
\end{itemize}
\eprop
\begin{proof}
(a) By the definition of $Q$, the optimal transport plan from $\Pn$ to $Q$ is
such that each support point $s\in S$ is received by all $x \in \Xset$
nearest to $s$ compared to other points in $S$. So,

\begin{align}
\Wasserstein^2(\Pn,Q) 
&= \sum_{\bm x\in \Xset} w(\bm x) d(\bm x, \bm G)^2  \\
&= \sum_{\bm x \in \Xset} w(\bm x) \sum_{j=1}^{m} d(\pi_{S_j}\bm x, \bm C_j)^2 \\ 
&= \sum_{j=1}^{m} \sum_{\bm z\in \Xset_j} w_j(\bm z) d(\bm z,\bm C_j)^2 \\
& = \sum_{j=1}^{m} \Wasserstein^2(M_j,\Pn_j) \\
& \leq \alpha \sum_{j=1}^{m} \Wasserstein^2(M_j^*,\Pn_j)  \\
& \leq \alpha \cdot \Wasserstein^2(M^*,\Pn).
\end{align}

The second to last inequality is due to the $\alpha$-approximation of $\wkmeans_1$,
and condition that $|\textrm{supp}(M_j)| \geq |\textrm{supp}(M_j^*)|$.
The last inequality follows from Proposition \ref{prop:relateOpt}.
By the triangle inequality of $\Wasserstein$, 

\begin{align}
    \Wasserstein(\Pn,P) 
    & \leq \Wasserstein(\Pn,Q) + \Wasserstein(Q, P) \\
    & \leq \Wasserstein(\Pn,Q) + \sqrt\gamma \cdot \Wasserstein(Q,M^*) \\
    & \leq \Wasserstein(\Pn,Q) + \sqrt \gamma (\Wasserstein(Q,\Pn) +
    \Wasserstein(\Pn,M^*)) \\
    & = (1+\sqrt\gamma) \Wasserstein(\Pn,Q) + \sqrt \gamma \cdot \Wasserstein(\Pn,M^*) \\
    & \leq (1+\sqrt\gamma) \sqrt \alpha \Wasserstein(\Pn,M^*) + \sqrt \gamma \cdot
    \Wasserstein(\Pn,M^*) \\
    & = \bigl( \sqrt\alpha+\sqrt\gamma+\sqrt{\alpha\gamma} \bigr) \cdot \Wasserstein(M^*,\Pn). 
\end{align}

The second inequality is due to the fact that $\wkmeans_2$ has approximation
ratio $\gamma$; the first and third  are the triangle
inequality. We conclude the proof.

    (b) We only need to construct an example of $\Pn$ for the case $d = m$.
Although $\Pn$ as an input to the algorithm is a discrete measure, for the
purpose of this proof it suffices to take $\Pn$ to be the uniform distribution
on $[0,1]^m$ (which can be approximated arbitrarily well by a discrete measure).
It is simple to verify that if $k_0=k^{1/m}$ is a natural number, then $M^*$ is a
uniform distribution on the regular grid of size $k_0$ in each dimension. It
follows that $\Wasserstein^2(\Pn,M^*) \leq \frac{m}{12k_0^3} =
\frac{m}{12k^{3/m}}$.  The grid points $\bm G$ range over the set
$S:=[1/(2\kappa),1-1/(2\kappa)]^m$. Moreover, $Q$ is a uniform distribution on
$\bm G$. Now $P$ is the outcome of line (11) so the support of
$P$ must lie in the convex hull $S$ of $\bm G$. The cost of each unit
mass transfer from an atom in the complement of set $[1/(4\kappa),
1-1/(4\kappa)]^m$ to one in $S$ is at least $(1/4\kappa)^2$, so 
$\Wasserstein^2(\Pn,P) \geq (1/4\kappa)^2\cdot[1- (1-1/(2\kappa))^m]$. 
We note $(1-1/(2\kappa))^m < e^{-m/2\kappa}$ to conclude the proof.
\end{proof}

The condition of part (a) is satisfied, for instance, by setting $\kappa = k$. In practice,
$\kappa < k$ may suffice. Moreover, part (b) dictates that $\kappa$ must grow with $k$
appropriately for our algorithm to maintain a constant approximation guarantee.
Since solution $\bm C$ has cost $L(\bm X, \bm C, w) = \Wasserstein^2(P, \Pn)$,
and $\opt(\bm X, w) = \Wasserstein^2(M^*, \Pn)$, the following theorem is immediate from
Prop.~\ref{W-bounds}(a).
\begin{thm}\label{thm:approx:bound}
Suppose $\wkmeans_1$ and $\wkmeans_2$ have approximation ratios $\alpha$
and $\gamma$. Then by choosing $\kappa = k$, the solution $\bm C$ 
given by Rk-means has the following guarantee:
$L(\bm X,\bm C,w) \leq
    (\sqrt \gamma+ \sqrt \alpha + \sqrt{\gamma \alpha})^2 \opt(\bm X,w).$

Specifically, if both sub-problems are solved optimally ($\alpha=\gamma=1$),
Rk-means is a $9$-approximation.
\end{thm}
\paragraph{Regularized Rk-means}
It is possible to extend our approach to accommodate regularization techniques. 
This can be useful when the data are very high dimensional~\cite{Sun-Wang-2012,Witten-Tibshirani-10}.
Thus, the clustering formulation can be expressed as a regularized optimal transport problem:
 $M^* = \argmin \; \Wasserstein^2(M, \Pn) + \Omega(M)$
where the optimization is over the space of discrete measures $M$ that have $k$
support points (the set $\bm C$ of $k$ centroids), and the regularizer $\Omega(M)\geq 0$
typically decomposes over the $m$-partition of variables:
$\Omega(M) = \sum_{j=1}^{m} \Omega_j(M_j)$. For instance, $\Omega_j(M_j)$ may be taken
to be a multiple of the $\ell_1$ norm of $M_j$'s supporting atoms (e.g., group lasso penalty).
The algorithm has the same three steps as before, with some modification in (1') and (3'):
\begin{itemize}
    \item [(1')] For each $j \in [m]$, pick $M_j$ to be the ($\alpha$-approximate) minimizer of
        $\Wasserstein^2(M_j, \Pn_j) + \Omega_j(M_j)$,
        where $\bm C_j = \textrm{supp}(M_j)$ is the support of $M_j$ 
and $|C_j| = \kappa$ (line 7). 
    \item [(3')] Finally, return $P$ which is the measure with exactly $k$ support points
    in $\mathbb{R}^d$ that ($\gamma$-approximately)
        minimizes $\Wasserstein^2(P,Q) + \Omega(P)$ (line 11). 
\end{itemize}

\bprop
\label{R-bounds}
If $\kappa \geq |\textrm{supp}(M_j^*)|$ for all $j \in [m]$, then
\begin{align}
\frac{\Wasserstein^2(\Pn,P)+\Omega(P)}{\Wasserstein^2(\Pn,M^*)+\Omega(M^*)}
&\leq 2\alpha + 4\gamma + 4\alpha\gamma.
\end{align}
\eprop
\begin{proof}
As before the optimal transport plan from $\Pn$ to $Q$ is
such that each support point $s\in S$ is received by all $x \in \Xset$
nearest to $s$ compared to other points in $S$. So,

\begin{align}
\Wasserstein^2(\Pn,Q) + \Omega(M)
& = \sum_{j=1}^{m} \Wasserstein^2(M_j,\Pn_j) + \Omega(M) \\
& \leq \alpha \sum_{j=1}^{m} (\Wasserstein^2(M_j^*,\Pn_j) + \Omega_j(M_j^*)) \\
& \leq \alpha (\Wasserstein^2(M^*,\Pn) + \Omega(M^*)).
\label{imb}
\end{align}

The second to last inequality is due to the $\alpha$-approximation of (regularized) $\wkmeans_1$,
and condition that $|\textrm{supp}(M_j)| \geq |\textrm{supp}(M_j^*)|$.
The last inequality follows from Proposition \ref{prop:relateOpt} and the definition of
$\Omega$.
By the triangle inequality of $\Wasserstein$, as before

\begin{align}
  \Wasserstein(\Pn,P)
    & \leq \Wasserstein(\Pn,Q) + \Wasserstein(Q, P) \\
    & \leq \Wasserstein(\Pn,Q) +
     \sqrt{\gamma \Wasserstein^2(Q,M^*) + \gamma \Omega(M^*) - \Omega(P)} \\
    & \leq \Wasserstein(\Pn,Q) \\
    &+      \sqrt{2\gamma \Wasserstein^2(\Pn,Q) + 2\gamma \Wasserstein^2(\Pn,M^*) + \gamma
     \Omega(M^*) - \Omega(P)}.
\end{align}

Hence, by Cauchy-Schwarz and combining with~\eqref{imb} we obtain
\begin{align}
\Wasserstein^2(\Pn,P) 
& \leq 2 \biggr \{(1+2\gamma)\Wasserstein^2(\Pn,Q)
+ 2\gamma\Wasserstein^2(\Pn,M^*)+ \gamma\Omega(M^*)-\Omega(P)\biggr \} \\
& \leq (2\alpha+4\gamma + 4\alpha\gamma) \Wasserstein^2(\Pn,M^*) 
+ (2\alpha + 2\gamma + 4\alpha\gamma)\Omega(M^*) - (2+4\gamma)\Omega(M) -2\Omega(P).
\end{align}

The conclusion is immediate by noting that $\Omega$ is a non-negative function.
\end{proof}

\comment{
\begin{proof}
As before the optimal transport plan from $\Pn$ to $Q$ is
such that each support point $s\in S$ is received by all $x \in \Xset$
nearest to $s$ compared to other points in $S$. So,

\begin{align}
&\Wasserstein^2(\Pn,Q) + \Omega(M)
= \sum_{j=1}^{m} \Wasserstein^2(M_j,\Pn_j) + \Omega(M)
\nonumber  \leq \alpha \sum_{j=1}^{m} (\Wasserstein^2(M_j^*,\Pn_j) + \Omega_j(M_j^*)) \\
& \leq \alpha (\Wasserstein^2(M^*,\Pn) + \Omega(M^*)).
\label{imb}
\end{align}

The second to last inequality is due to the $\alpha$-approximation of (regularized) $\wkmeans_1$,
and condition that $|\textrm{supp}(M_j)| \geq |\textrm{supp}(M_j^*)|$.
The last inequality follows from Proposition \ref{prop:relateOpt} and the definition of
$\Omega$.
By the triangle inequality of $\Wasserstein$, as before

\begin{align}
    \Wasserstein(\Pn,P) 
    & \leq \Wasserstein(\Pn,Q) + \Wasserstein(Q, P) \leq \Wasserstein(\Pn,Q) +
     \sqrt{\gamma \Wasserstein^2(Q,M^*) + \gamma \Omega(M^*) - \Omega(P)}  \nonumber\\
     & \leq \Wasserstein(\Pn,Q) +
     \sqrt{2\gamma \Wasserstein^2(\Pn,Q) + 2\gamma \Wasserstein^2(\Pn,M^*) + \gamma \Omega(M^*) - \Omega(P)}  \nonumber.
\end{align}

Hence, by Cauchy-Schwarz and combining with~\eqref{imb} we obtain

\begin{align}
\Wasserstein^2(\Pn,P)
& \leq 2 \biggr \{(1+2\gamma)\Wasserstein^2(\Pn,Q) + 2\gamma\Wasserstein^2(\Pn,M^*)+
\gamma\Omega(M^*)-\Omega(P)\biggr \} \\
& \leq (2\alpha+4\gamma + 4\alpha\gamma) \Wasserstein^2(\Pn,M^*) + (2\alpha + 2\gamma + 4\alpha\gamma)\Omega(M^*) -
(2+4\gamma)\Omega(M) -2\Omega(P).
\end{align}
\end{proof}
}

If both subproblems for regularized $k$-means can be solved optimally, our method yields
a 10-approximation on the penalized $\Wasserstein^2$ objective.
We conclude by noting that our technique extends easily 
to the $\Wassersteinp^p$ objective for any $p\geq 1$, but the approximation ratio 
will be changed according to $p$.

\section{Leveraging structures in relational data}
\label{sec:structures}

We now explain the ``relational'' part of the Rk-means algorithm, where we
exploit relational structures in the data and the FEQ to 
achieve significant computational savings.
Three classes of relational structures prevalent in RDBMSs are
(a) {\it categorical variables}, (b) 
{\it functional dependencies} (FDs),
and
(c) the topology of the FEQ.
We exploit these structures to carefully select the partition $S_1 \cup \cdots
\cup S_m$ to use for Rk-means, to compute the marginal sub-problems $(\bm X_j, w_j)$,
the components $\bm C_j$ of the coreset $\bm G$, and the grid weight
$w_{\textsf{grid}}$ {\em without} materializing the entire coreset $\bm G$.
When selecting partitions, there are two
competing criteria:
first, we need a partition so that the approximation ratio $\alpha$ 
for $\wkmeans_1$ is as small as possible.  For example, if 
$|S_j|=1$ for all $j$, so $m=d$, then we can apply the
well-known optimal solution for $k$-means in $1$ dimension using dynamic
programming in $O(n^2k)$ time~\cite{wang2011ckmeans}; this then provides
$\alpha=1$.  On the other hand, we want the remainder of algorithm to be fast by
keeping the size of the grid $\bm G$, namely $|\bm G| \le \kappa^m$, small.

\subsection{Categorical variables}

Real-world relational database queries typically involve many categorical
variables (e.g., {\small \sf color}, {\small \sf month}, or {\small \sf city}).
In practice, practitioners may endow non-uniform weights for different
categorical variables, or categories \cite{Huang-1998}. In terms of 
representation, the most common way to deal with categorical variables 
is to one-hot encode them, whereby a categorical feature such as ${\small \sf
city}$ is represented by an indicator vector
\begin{align}
\label{catfeature}
    \bm x_{\city} = \begin{bmatrix}
    \bm 1_{\city=c_1} &
    \bm 1_{\city=c_2} &
    \cdots
    \bm 1_{\city=c_L}
\end{bmatrix}
\end{align}
where $\{c_1,\dots,c_L\}$ is the set of cities occuring in the data.
The subspace associated with these indicator vectors is known as the {\em
categorical subspace} of a categorical variable.
This one-hot representation substantially increases the data matrix size via
an increase in the {\em dimensionality} of the data.
For example, a dataset of about $30$ mostly categorical features with hundreds
or thousands of categories for each feature will have its dimensionality
exploded to the order of thousands with one-hot encoding.
Fortunately, this is not a problem --- by treating each categorical variable 
as a subset of the partition, the \emph{weighted} $k$-means subproblem within 
a categorical subspace is solvable efficiently and optimally.  

This optimal solution can be computed in the same time it takes to find the
number of points in each category, which is a vast improvement on either an
optimal dynamic program or Lloyd's algorithm.  Furthermore, it helps keep $m$ as
low as the number of database attributes in the query.

Consider a weighted $k$-means subproblem solved by $\wkmeans_1$ defined on a
categorical subspace induced by a categorical feature $K$ that has $L$
categories.  Then, the instance is of the form $(\bm I, v)$, where $\bm I$ is
the collection of $L$ indicator vectors $\bm 1_e$, one for each element $ e\in
\dom(K)$.  (One can think of $\bm I$ as the
identity matrix of order $L$.)  Define the weight function $v$ as
\begin{align}
    v(\bm 1_e) = \sum_{\bm x \in \bm X, \bm x_K=e} w(\bm x).
\end{align}

For any set $F \subseteq \dom(K)$,
let $\bm v_F$ denote the vector $(v(\bm 1_e))_{e\in F}$.
Also, $\|\bm v_F\|_1$ and $\|\bm v_F\|_2$ denote
the $\ell_1$ and $\ell_2$ norm, respectively.
It is useful to rewrite the categorical weighted $k$-means problem:

\bprop \label{prop:min:max}
The categorical weighted $k$-means instance $(\bm I,v)$
admits the following optimization objective:
\begin{align}
    \opt(\bm I,v) =
        \norm{\bm v}_1 -
        \max_{\calF}~
        \sum_{F \in \calF} \frac{\norm{\bm v_{F}}_2^2}{\norm{\bm v_{F}}_1}, \label{eqn:catobj}
\end{align}
where $\calF$ ranges over all partitions of $\dom(K)$ into $k$ parts.
\eprop
\bp
First, consider a subset $F \subseteq \dom(K)$ of the categories; the
centroid $\bm \mu$ of (weighted) indicator vectors $\bm 1_e$, $e \in F$, can be written
down explicitly:
\begin{align}
    \mu_e &= \begin{cases}
    0 & e \notin F\\
    \frac{v_e}{\norm{\bm v_F}_1} & v \in F,
    \label{eqn:mu}
\end{cases}
\end{align}
The weighted sum of squared distances between $\bm 1_e$ for all $e\in F$ to $\bm \mu$ is
\begin{align*}
   \sum_{e \in F}(\norm{\bm \mu}_2^2 - \mu_e^2 + (\mu_e-1)^2)v_e
    &=\frac{\norm{\bm v_{F}}_2^2}{\norm{\bm v_{F}}_1} + \sum_{e \in F} ( (\mu_e-1)^2- \mu_e^2)v_e \\
       &=\frac{\norm{\bm v_{F}}_2^2}{\norm{\bm v_{F}}_1} + \sum_{e \in F} ( -2\mu_e +1)v_e \\
    & = \norm{\bm v_{F}}_1 -\norm{\bm v_{F}}_2^2/\norm{\bm v_{F}}_1.
\end{align*}
Thus, the weighted $\kmeans$ objective takes the form
\begin{align}
     \min_{\calF}
        \sum_{F \in \calF} \left( \norm{\bm v_{F}}_1 - \norm{\bm
        v_{F}}_2^2/\norm{\bm v_{F}}_1 \right) 
        &= \norm{\bm v}_1 - \max_{\calF}~
        \sum_{F \in \calF} \norm{\bm v_{F}}_2^2/\norm{\bm v_{F}}_1,
\end{align}
which concludes the proof.
\ep

In~\eqref{eqn:catobj}, note that $\norm{\bm v}_1$ is the total weight of input points;
hence, we can equivalently solve the inner maximization  problem.
With the categorical weighted $k$-means objective in place,
we can derive the optimal clustering. To do so,
We next need the following elementary lemma.

\begin{lmm}
Suppose that $x, a_1,a_2,b_1,b_2 > 0$, $b_1^2 \geq a_1$, $b_2^2 \geq a_2$
and $x \geq \max\{a_1/b_1,a_2/b_2\}$.
Then $x+ \frac{a_1+a_2}{b_1+b_2} \geq \max \biggr \{\frac{x^2+a_1}{x+b_1} +
\frac{a_2}{b_2},
\frac{x^2+a_2}{x+b_2} + \frac{a_1}{b_1}\biggr \}.$
\label{lem:aux}
\end{lmm}
\begin{proof}
  It suffices to establish
  $x+ \frac{a_1+a_2}{b_1+b_2} \geq \frac{x^2+a_1}{x+b_1} + \frac{a_2}{b_2}$,
  or equivalently
       \[x- \frac{x^2+a_1}{x+b_1} \geq
        \frac{a_2}{b_2} - \frac{a_1+a_2}{b_1+b_2},\]

which can be simplified as
\begin{equation}
x(b_1+b_2+a_1/b_1-a_2/b_2) \geq a_1b_2/b_1 + a_2b_1/b_2.
\end{equation}
To verify this inequality, consider two cases. If $a_1/b_1 \geq a_2/b_2$, then
$LHS \geq x(b_1+b_2) \geq (a_2/b_2)b_1 + (a_1/b_1)b_2$.
On the other hand, if $a_2/b_2 > a_1/b_1$. Since $b_2-a_2/b_2 \geq 0$,~
\begin{align*}
LHS & \geq (a_2/b_2)(b_1+b_2+a_1/b_1 - a_2/b_2) \\
& = a_2b_1/b_2 + a_2 + a_1a_2/(b_1b_2) - a_2^2/b_2^2 \\
& = a_2b_1/b_2 + a_1b_2/b_1 + (b_2-a_2/b_2)(a_2/b_2-a_1/b_1) \\
& \geq a_2b_1/b_2 + a_1b_2/b_1.
\end{align*}
Thus the proof is complete.
\end{proof}

Then, the optimal solution to the categorical $k$-means instance is an immediate
consequence:

\begin{cor} \label{cor:struc}
Let $(e_1,\dots,e_L)$ be a permutation of $\dom(K)$ such that
$v_{e_1} \geq v_{e_2} \geq \ldots \geq v_{e_L}$.
Then for any $k \geq 2$ and any
$k$-partition $\calF$ of $\dom(K)$, there holds
\begin{equation*}
v_{e_1}+\ldots+v_{e_{k-1}} + \frac{\sum_{i=k}^{L} v_i^2}{\sum_{i=k}^{L} v_i}
\geq \sum_{F \in \calF} \frac{\norm{\bm v_{F}}_2^2}{\norm{\bm v_{F}}_1}.
\end{equation*}
\end{cor}
\begin{proof}
We prove the claim by induction on $k$. Let $F \in \calF$ be the set containing the
element $\{e_1\}$.  If there is only one element in $F$ then we apply the induction
hypothesis on the remaining terms. Otherwise, $F$ contains at least two elements.
Let $G \in \calF$ be an arbitrary element of $\calF$ where $G \neq F$.
Define $\calF'$ to be the partition obtained from $\calF$ by replacing
$(F,G)$ with $(\{e_1\}, F\cup G-\{e_1\})$. Then,  Lemma~\ref{lem:aux} can be applied
to get
\begin{eqnarray*}
\sum_{F\in \calF}\frac{\norm{\bm v_{F}}_2^2}{\norm{\bm v_{F}}_1}
\leq
\sum_{F\in \calF'}\frac{\norm{\bm v_{F}}_2^2}{\norm{\bm v_{F}}_1}.
\end{eqnarray*}
Induction on the tail $k-1$ terms completes the proof.
\end{proof}

Theorem~\ref{thm:categorical} below follows trivially from the above corollary.
Corollary~\ref{cor:struc} and the objective for $k$-means on a single attribute
in the equation of Proposition~\ref{prop:min:max} establishes precisely the structure of the 
optimal solution for data consisting of a single categorical variable.  

\bthm
Given a categorical weighted $k$-means instance, an optimal solution can
be obtained by putting each of the first $k-1$ highest weight indicator vectors
in its own cluster, and the remaining vectors in the same cluster.
\label{thm:categorical}
\ethm

This means that for a categorical variable with $L$ categories, we can compute the 
optimal clustering for the sub-problem in only $O(nL \log L)$ time.
The variable gives rise to a categorical subspace of size $|S_j| = L$.

\subsection{Functional dependencies}

Next, we address the second call to $\wkmeans_2$: its runtime is dependent on the
size of the grid $\bm G$, which can be up to $O(k^m)$, where $m$ is the number of
features from the input.
Databases often contain {\it functional dependencies} (FDs), which we can
exploit to reduce the size of $\bm G$.  An FD is a dimension whose value depends
entirely on the value of another dimension. For example, for a retailer dataset
that includes geographic information, 
one might encounter features such as
{\small \sf storeID}, {\small \sf zip}, {\small \sf city}, {\small \sf state},
and {\small \sf country}. Here, {\small \sf storeID} functionally determines
{\small \sf zip}, which determines {\small \sf city}, which in turns determines
{\small \sf state}, leading to {\small \sf country}.  This common structure is
known as an {\em FD-chain}, and appears often in real-world FEQs.

If we were to apply Rk-means without exploiting the FDs, the features {\small
\sf storeID}, {\small \sf zip}, {\small \sf city}, {\small \sf state}, and
{\small \sf country} would contribute a factor of $k^5$ to the grid size.
However, by using the FD structure of the database, we show
that only a factor of $5k$ is contributed to the grid size, because most of the
$k^5$ grid points $\bm g$ have $\wgrid(\bm g) = 0$ (see~\eqref{eqn:wg}).
More generally, whenever there is an {\em FD chain} 
including $p$ features, their overall contribution to the grid size
is a factor of $O(kp)$ instead of $O(k^p)$, and the grid points with non-zero
weights can be computed efficiently in time $O(kp)$.

\blmm
Suppose all $d$ input features are categorical and form an FD-chain. Then, the
total number of grid points $\bm g \in \bm G$ with non-zero $\wgrid$ weight is
at most $d(k-1)+1$.
\label{lmm:fd:chain}
\elmm
\bp
Suppose the features are $K_1,\dots,K_d$, where $K_i$ functionally
determine $K_{i+1}$, and
$\dom(K_i) = \{e^i_1, e^i_2, \cdots, e^i_{n_i}\}$. Without loss
of generality, we also assume that the elements in $\dom(K_i)$ are
sorted in descending order of weights:
\begin{align}
    w(\bm 1_{e^i_1}) \geq
    w(\bm 1_{e^i_2}) \geq
    \cdots
    \geq
    w(\bm 1_{e^i_{n_i}}).
\end{align}
From Corollary~\ref{cor:struc}, we know the set $\bm C_i$ of $k$ centroids of each of the
categorical subspace for $K_i$: there is a centroid
$\bm \mu^i_j = \bm 1_{e^i_j}$
for each $j \in [k-1]$, and then a centroid $\bm \mu^i_{k}$ of the rest of the indicator vectors.
The elements $e^i_j$ for $j \in [k-1]$ shall be
called {``heavy''} elements, and the rest are ``light'' elements.

Now, consider an input vector $\bm x = (x_1,\dots,x_d)$ where
$x_i \in \dom(K_i)$. Under one-hot encoding, this vector is mapped to a vector of indicator vectors $\bm 1_{\bm x} := (\bm 1_{x_1}, \cdots, \bm 1_{x_d})$. We need to answer the question: which grid point in $\bm G = \bm C_1 \times \cdots \times \bm C_d$ is $\bm 1_{\bm x}$ closest to? Since the $\ell_2^2$-distance is decomposable into component sum, we
can determine the closest grid point by looking at the
closest centroid in $\bm C_i$ for $\bm 1_{x_i}$, for
each $i\in [d]$.

If $x_i \in \{e^i_1,\dots, e^i_{k-1}\}$, then the corresponding one-hot-encoded version $\bm 1_{x_i}$
{\em is} itself one of the centroids in $\bm C_i$, and thus it is
its own closest centroid. Otherwise, the closest centroid to
$\bm 1_{x_i}$ is $\bm \mu^i_k$, because $\norm{\bm 1_{x_i}- \bm\mu^i_k}^2 < 2$, and $\norm{\bm 1_{x_i}- \bm \mu^i_j}^2=2$ for every $j \in [k-1]$.

Let $\bm\mu^i(x_i) \in \bm C_i$ denote the closest centroid in $\bm C_i$ to $\bm 1_{\bm x_i}$. The closest grid point to $\bm 1_{\bm x}$ is completely determined:
$\bm g = (\bm\mu^1(x_1), \cdots, \bm \mu^d(x_d))$.
Furthermore, let $i \in [d]$ denote the smallest index such that $x_i$ is heavy. Then, we can write $\bm g$ as
\begin{align}
    \bm g = (\bm\mu^1_k, \cdots, \bm\mu^{i-1}_k,
    \bm 1_{x_i}, \bm\mu^{i+1}(x_{i+1}), \cdots,
    \bm \mu^{d}(x_d))
\end{align}
Note that once $x_i$ is fixed, due to the FD-chain the {\em entire} suffix $(\bm 1_{x_i}, \bm\mu^{i+1}(x_{i+1}), \cdots,
\bm \mu^{d}(x_d)$ of $\bm g$ is determined. Hence, the number
of different $\bm g$s can only be at most $d(k-1)+1$:
there are $d+1$ choices for $i$ (from $0$ to $d$), and $k-1$ choices for $x_i$ if $i >0$.
\ep

Theorem~\ref{thm:fds:coreset} below follows trivially from the above lemma, because
the $\ell_2^2$-distance is the sum over the $\ell_2^2$-distances of the
subspaces.

\bthm
Suppose all $d$ input features can be partitioned into $m$ FD-chains of size
$d_1, \cdots, d_m$, respectively. Then, the number of grid points $\bm g \in \bm
G$ with non-zero $\wgrid$ weight is bounded by $\prod_{i=1}^m (1+d_i(k-1))$.
Furthermore, the set of non-zero weight grid points can be computed in time
$\tilde O(\prod_{i=1}^m (1+d_i(k-1)))$.
\label{thm:fds:coreset}
\ethm

Note that in the above theorem, if there was {\em no} FD, then $d$ features each form
their own chain of size $1$, in which case $\prod_{i=1}^m (1+d_i(k-1)) = k^m$;
thus, the theorem strictly generalizes the no-FD case.

\subsection{Query structure}

Finally, we explain how the FEQ's structure can be exploited to speed up the computation of
subproblems, the grid, and grid weights. In particular, we make use of recent advances in
relational query evaluation algorithms~\cite{DBLP:journals/sigmod/Khamis0R17,
DBLP:conf/pods/KhamisNR16,
DBLP:conf/pods/000118,
DBLP:journals/sigmod/OlteanuS16}. The $\InsideOut$ algorithm from the $\faqs$ framework in
particular~\cite{DBLP:journals/sigmod/Khamis0R17} allows us to compute the grid weights
without explicitly the grid points.

For concreteness, we describe the steps of Rk-means as implemented in the database, noting
the additional speedups we can get over the description in
Algorithm~\ref{alg:kmeans}.

\noindent \textbf{Step 1} (lines 5 and 6).  {\it Project $\Xset$ into each
subspace $S_j$ and compute the weight $w$ of each point.}

In a relational database, the projected sets $\Xset_j$ already exist in normalized 
form~\cite{DBLP:books/aw/AbiteboulHV95}, and thus they and their marginal weights
can be computed highly efficiently.
This step perfectly aligns with 
our strategy of picking the partition $S_1 \cup \cdots \cup S_m$ to match the database schema!

\noindent \textbf{Step 2} (line 7).  {\it Find $\kappa$ centroids in each
subspace $S_j$.}

If the subspace $S_j$ corresponds to a single continuous variable, we can solve
the one-dimensional $k$-means problem quickly and
optimally~\cite{wang2011ckmeans}, and if the subspace corresponds to a
categorical feature, then it is solved trivially (and optimally) using
Theorem~\ref{thm:categorical}.

\noindent \textbf{Step 3} (lines 9 and 10).  {\it Construct the coreset
$\bm G$ and the associated weights $\wgrid$.}

When constructing $\bm G$, it is unnecessary to represent any points in $\bm G$
that have zero weight.  We use \InsideOut~\cite{DBLP:conf/pods/KhamisNR16} to efficiently compute nonzero
weights, and then extract only those grid points in $\bm G$ with nonzero weight
from the database.

\noindent \textbf{Step 4} (line 11).  {\it Cluster the weighted coreset $\bm
G$.}

We use a modified version of Lloyd's weighted $k$-means that
exploits the structure of $\bm G$ and sparse representation of categorical
values to speed up computation

We discuss the optimization and acceleration of Step 4 of the Rk-means
implementation in more details here.  Recall that the
categorical subspace $\kmeans$ problem is solved trivially using
Theorem~\ref{thm:categorical}, where we sort all the weights, and the heaviest
$k-1$ elements form their own centroid, while the remaining vectors are
clustered together (the ``light cluster'').

If $S_j$ is a categorical subspace corresponding to a categorical variable $K$ where
$\dom(K) = \{e_1,\dots,e_L\}$. Without loss of generality, assume $w(\bm
1_{e_1}) \geq \cdots \geq w(\bm 1_{e_L})$, then the centroid of the light
cluster is an $L$-dimensional vector $\bm c = (s_e)_{e \in \dom(K)}$
\begin{align}
    s_{e_i} := \begin{cases}
        0 & i \in [k-1] \\
        \frac{w(\bm 1_{e_i})}
        {\sum_{j=k}^L w(\bm 1_{e_j})} & i \geq k
    \end{cases}
\end{align}
This encoding is sound and space-inefficient.

Remember also that Step 4 clusters the coreset $\bm G$ using a modified version
of Lloyd's weighted $k$-means that exploits the structure of $\bm G$ and sparse
representation of categorical values. We show how to improve the distance
computation $\norm{\bm c_j- \bm \mu_j}^2$ for sub-space $S_j$, where $\bm c_j$
and $\bm \mu_j$ are the $j$-th components of a grid point and respectively of a
centroid for $\bm G$. Since this subspace corresponds to a categorical variable
$K$ with, say, $L_j$ categories, it is mapped into $L_j$ sub-dimensions. Let
$\bm c_j=[s_1,\ldots,s_{L_j}]$ and $\bm \mu_j=(t_1,\ldots,t_{L_j})$.  Using the
explicit one-hot encoding of its categories, we would need $O(L_j)$ time to
compute $\norm{\bm c_j-\bm \mu_j}^2=\sum_{\ell\in [L_j]} (s_\ell - t_\ell)^2$.
We can instead achieve $O(1)$ time as shown next.  There are $k$ distinct values
for $\bm c_j$ by our coreset construction, each represented by a vector of size
$L_j$ with one non-zero entry for $k-1$ of them and $L_j-k+1$ non-zero entries
for one of them.

If $\bm c_j = \bm 1_e$ is an indicator vector for some element $e \in K$
($e$ is one of the $k-1$ heavy categories), then
\begin{align}
    \norm{\bm c_j-\bm\mu_j}^2 &=\norm{\bm 1_e-\bm\mu_j}^2 = 1-2t_e +\norm{\bm\mu_j}^2.
    \label{eqn:one}
\end{align}
If $\bm c_j$ is a light cluster centroid,
\begin{align}
    \norm{\bm c_j-\bm\mu_j}^2 &= \norm{\bm c_j}^2+\norm{\bm \mu_j}^2-2\inner{\bm
    c_j,\bm\mu_j}.
    \label{eqn:two}
\end{align}
In~\eqref{eqn:one}, by pre-computing $\norm{\bm\mu_j}^2$ we only spend $O(1)$-time
per heavy element $e$.
In~\eqref{eqn:two}, by also pre-computing $\norm{\bm c_j}^2$ and $\inner{\bm c_j, \bm\mu_j}$,
and by noticing that $\bm c_j$ is $(L_j-k+1)$-sparse, we spend $O(L_j-k)$-time here.
Overall, we spend time $O(L_j)$ for computing $\norm{\bm c_j-\bm\mu_j}^2$
per categorical dimension, modulo the precomputation time.

Step 4 thus requires $O(|\bm G|mk+\sum_{j\in[m]} L_jk) = O(|\bm G|mk + Dkm)$ per
iteration, whereas a generic approach would take time
$O(\sum_{j\in[m]}|\bm G|kL_j)=O(|\bm G|Dkm)$. Our modified weighted $k$-means algorithm
thus saves a factor proportional to the total domain sizes of the categorical variables, 
which may be as large as $D$.

\subsection{Runtime analysis}
\label{subsec:runtime}

We compare Rk-means to the standard setting of first extracting the matrix $\Xset$ from the 
database and then perform clustering on $\Xset$ directly.  
The precise runtime statement requires defining a few parameters such as ``fractional
hypertree width'' and ``fractional edge cover number'' of the FEQ, which we briefly covered
in Section~\ref{sec:database_queries}. 
Hence, we state the main thrust of our runtime result:

\bthm
There are classes of feature extraction queries (FEQs) for which the runtime of Rk-means is 
asymptotically less than
$|\bm X|$, and the ratio between $|\bm X|$ and the runtime of Rk-means can be a polynomial in
$N$, the size of the largest input relation.
\label{thm:runtime:rkmeans}
\ethm
\bp[Proof of Theorem~\ref{thm:runtime:rkmeans}]
Let $N$ denote the maximum number of tuples in any input relation of
the FEQ, 
$|\Xset|$ the number of tuples in the data matrix,
$\fhtw$ the {\em fractional hypertree width} of the FEQ 
$t$ the number of iterations of Lloyd's algorithm,
$d$ denote the number of features pre-one-hot encoding,
$r$ number of input relations to the FEQ,
$D$ the real dimensionality of the problem after one-hot-encoding.

We analyze the time complexity for each of the four steps of the Rk-means
algorithm.  

Step 1 projects $\Xset$ into each subspace
$S_j$ and compute the total weight of each projected point:
\begin{align}\label{eq:faq-one-dimension-count}
    \forall j\in[d]: w_j(\bm x_{S_j})
    &:= \sum_{{\bm x}_{[d]\setminus\{S_j\}}} \prod_{F\in \calE}R_F({\bm x}_F)
\end{align}
Each of the $d$ FAQs~\eqref{eq:faq-one-dimension-count} in Step 1
can be computed in time $\tilde O(rd^2N^\fhtw)$ using $\InsideOut$, as we have reviewed in
Section~\ref{sec:database_queries}.

In Step 2, the optimal clustering in each dimension takes time $\tilde O(L_j)$ 
for each categorical variable $j$ (whose domain size is $L_j$,
and $O(kN^2)$ for each continuous variable, with an overall
runtime of $O(kdN^2)$.

Step 3 constructs $\bm G$, whose size is bounded by $|\bm X|$ 
and by the FD result of Theorem~\ref{thm:fds:coreset}.
In practice, this number can be much smaller
since we skip the data points whose weights are zero.
To perform this step we construct a tree decomposition of FEQ with with equal $\fhtw$
(this step is data-independent, only dependent on the size of FEQ).
Then, from each value $x_j$ of an input variable $X_j$, we determine its centroid $c(x_j)$
which was computed in step 2. By conditioning on combinations of $(c_1,\dots,c_j)$,
we can compute $w_{\textsf{grid}}$ one for each combiation in $\tilde O(dN^\fhtw)$-time,
for a total run time of $\tilde O(rd|\bm G|N^\fhtw)$.

Step 4 -- as analyzed in Section~\ref{subsec:runtime} --
clusters $\bm G$ in time $O((|\bm G| + D)kmt)$, where $t$ is the
number of iterations of $k$-means used in Step 4.
The most expensive computation is due to the one-dimensional clustering for the continuous 
variables and the computation of the coreset.

To compare the total runtime with $|\bm X|$, we only need to note that
$|\bm X|$ can be as large as $N^{\rho^*}$, where
$\rho*$ is the {\em fractional edge covering number} of the FEQ's hypergraph
\cite{DBLP:conf/pods/000118}.
Depending on the query, $\rho^*$ is always at least $1$, and can be as large as the number
of features $d$. 
Furthermore, there are classes of queries where $\fhtw$ is bounded by a constant, 
yet $\rho^*$ is unbounded~\cite{DBLP:journals/jacm/Marx13}.
This means, for classes of FEQs where $\rho^* > \max\{\fhtw,2\}$ the ratio between
$|\bm X|$ and Rk-means's runtime will be $\tilde Omega(N^{\rho^* - \max\{\fhtw,2\}}/t)$,
which is unbounded.
\ep

The key insight to read from this theorem is that Rk-means can, in principle, run faster
than simply exporting the data matrix, without even running {\em any} clustering algorithm
(be it sampling-based, streaming, etc.).
Of course, the result only concerns a class of FEQs ``on paper''.
Section~\ref{sec:experiments} examines real FEQs, which also demonstrate Rk-means's runtime
superiosity.

For reference,
we compare the asymptotic runtime of Rk-means to the standard implementation
of Lloyd's algorithm. 
The standard implementation contains two steps: (1) compute the 
one-hot-encoded data
matrix $\bm X$, and (2) run Lloyd's algorithm on $\bm X$. The first step, materializing $\bm
X$, takes time $\tilde O(rd^2 N^\fhtw + D|\bm X|)$. The second step, running Lloyd algorithm,
takes time $\tilde O(tkD|\bm X|)$, as is well known.
Thus, the standard approach takes time $\tilde O(rd^2 N^\fhtw + tkD|\bm X|)$.

\begin{figure}
    \begin{table}[H]
      \begin{tabular}{|l|r|r|r|}\hline 
        & Retailer & Favorita  & Yelp \\\hline
        Relations  & 5 & 6 & 6\\
        Attributes  & 39 & 15 & 25 \\
        One-hot Enc. & 95 & 1470 & 1617 \\\hline
        \# Rows in $\bm D$   & 84M & 125M & 8.7M \\
        Size of $\bm D$     & 1.5GB & 2.5GB & 0.2GB \\\hline
        \# Rows in $\bm X$ & 84M & 127M & 22M \\
        Size of $\bm X$ & 18GB& 7GB & 2.4GB \\\hline
        \multicolumn{4}{|c|}{\# Rows in Coreset $\bm G$}\\\hline
        $\kappa$ = 5  & 1.43M& 14.94K & 2.69M \\
        $\kappa$ = 10 & 9.58M & 85.88K & 11.71M \\
        $\kappa$ = 20 & 38.16M  & 632.5K & 11.89M \\
        $\kappa$ = 50 & 73.75M & 7.87M & 12.46M \\ \hline
      \end{tabular}
      \caption{Statistics for the input database $\bm D$, data matrix
        $\bm X$, and coresets $\bm G$ for the three dataset.}
      \label{tbl:dataset:stats}
    \end{table}
\end{figure}

\section{Experimental results}
\label{sec:experiments}

We empirically evaluate the performance of Rk-means on three real datasets for
three sets of experiments: (1) we break down and analyze the performance of
each step in Rk-means; (2) we benchmark the performance and approximation of
Rk-means against mlpack~\cite{mlpack2018} (v. 3.1.0), a fast C++ machine learning library;
and (3) we evaluate the performance and approximation of Rk-means for setting
$\kappa < k$; i.e., different number of clusters for Steps 2 and 4.

The experiments show that the coresets of Rk-means are often significantly
smaller than the data matrix. As a result, Rk-means can scale easily to large
datasets, and can compute the clusters with a much lower memory footprint than
mlpack.  When $\kappa = k$, Rk-means is orders-of-magnitude faster than the
end-to-end computation for mlpack---up to 115$\times$.  Typically, the
approximation level is very minor.  In addition, setting $\kappa < k$ can lead
to further performance speedups with only a moderate increase in approximation,
giving over 200$\times$ speedup in some cases.

{\bf Experimental Setup.} We prototyped Rk-means as part of an engine
designed to compute multiple FAQ expressions efficiently. Rk-means is implemented
in multithreaded C++11 and compiled with {\small\texttt -O3} optimizations; this
makes mlpack a comparable implementation.  All experiments were performed on an
AWS {\tt x1e.8xlarge} system, which has 1 TiB of RAM and 32 vCPUs.  All
relations given were sorted by their join attributes.

To construct the data
matrix that forms the input to mlpack, we use PostgreSQL ({\small \tt
  psql}) v. 10.6 to evaluate the FEQ.  The seminal $k$-means++
algorithm~\cite{Arthur-2007} is used for initializating the $k$-means cluster.
We run Rk-means and mlpack + {\small\texttt psql}
five times and report the average approximation and runtime.  The timeout for
all experiments was set to six hours (21,600 seconds) per trial.  Our runtime results omit
data loading/saving times.  Note that for mlpack + {\small\texttt psql},
{\small\texttt psql} must export $\bm X$ to disk, and then mlpack must then read
it from disk.  Rk-means has no need to do this, and thus the runtime numbers are
skewed in mlpack's favor.  This skew may be significant: loading and saving a
large CSV file may take hours in some cases.

{\bf Datasets.} We use three real datasets: (1) \textit{Retailer} is used by a
large US retailer for sales forecasting; (2) \textit{Favorita}~\cite{favorita}
is a public dataset for retail forecasting; and (3) \textit{Yelp} is from the
public Yelp Dataset Challenge~\cite{yelpdataset} and used to predict users'
ratings of businesses.  Table~\ref{tbl:dataset:stats} presents key statistics
for the three datasets, including the size of data matrix $\bm X$ and the
coreset $\bm G$ for each dataset and different $\kappa$-values.  $|\bm G|$ is
highly data dependent. For {\it Favorita}, $\bm G$ is orders-of-magnitude
smaller than the data matrix.  For {\it Retailer}, when $k=20$ and $k=50$,
$|\bm G|$ approaches $|\bm X|$, but Rk-means is still able to provide a speedup.
Additional dataset details are given below.

{\it Retailer} has five relations: \textit{Inventory} stores the number of
inventory units for each date, location, and stock keeping unit (sku);
\textit{Location} keeps for each store: its zipcode, the distance to the closest
competitors, and the type of the store; \textit{Census} provides 14 attributes
that describe the demographics of a given zipcode, including population size or
average household income; \textit{Weather} stores statistics about the weather
condition for each date and store, including the temperature and whether it
rained; \textit{Items} keeps track of the price, category, subcategory, and
category cluster of each sku. 

{\it Favorita} has six relations: \textit{Sales} stores the number of units sold
for items for a given date and store, and an indicator whether or not the unit
was on promotion at this time; \textit{Items} provides additional information
about the skus, such as the item class and price; \textit{Stores} keeps
additional information on stores, like the city they are located it;
\textit{Transactions} stores the number of transaction for each date and store;
\textit{Oil} provides the oil price for each date; and \textit{Holiday}
indicates whether a given day is a holiday. The original dataset gave the
{\small\tt units\_sold} attribute with a precision of three decimals
places. This resulted in a very many distinct values for this attribute, which
has a significant impact on the Step 2 of the Rk-means algorithm. We decreased
the precision for this attribute to two decimal places, which decreases the
number of distinct values by a factor of four. This modification has no effect
on the final clusters or their accuracy.

{\it Yelp} has five relations: \textit{Review} gives the review rating that a
user gave to a business and the date of the review; \textit{User} provides
information about the users, including how many reviews the made, when they
join, and how many fans they have; \textit{Business} provides information about
the businesses that are reviewed, such as their location and average rating;
\textit{Category} provide information about the categories, i.e. Restaurant, and
respectively attributes of the business, {\it Attributes} is an aggregated
relation, which stores the number of attributes (i.e., open late) that have
been assigned to a business. A business can be categorized in many ways, which
is the main reason why the size of the join is significantly larger than the
underlying relations.

\begin{figure*}
  \input{rkmeans_breakdown}
  \caption{Breakdown of the compute time of Rk-means for each step of the
    algorithm with $\kappa = k$. The time to compute $\bm X$ is provided as
    reference.}
  \label{fig:rkmeans:breakdown}
\end{figure*}
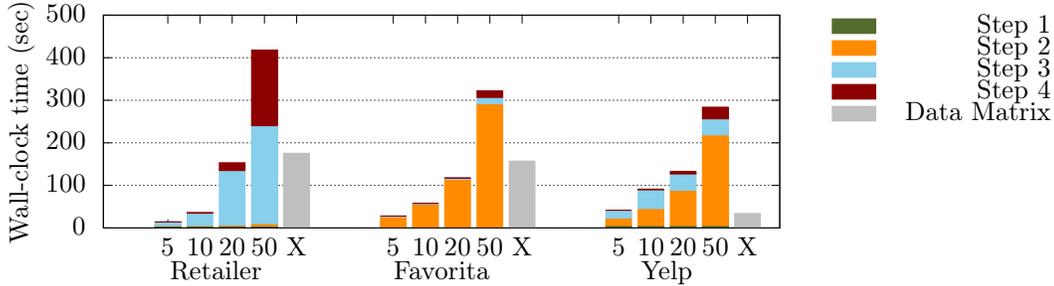

{\bf Breakdown of Rk-means.} Figure~\ref{fig:rkmeans:breakdown}
shows the time it takes Rk-means to cluster the three datasets for different
values of $k$ with $\kappa = k$. The total time is broken down into the four steps of the
algorithm from Section~\ref{sec:structures}. We provide the time it takes
{\small \tt psql} to compute $\bm X$ as reference (gray bar).  In many cases,
Rk-means can cluster {\it Retailer} and {\it Favorita} faster than it takes
{\small \tt psql} to even compute the data matrix. The relative performance of
the four steps is data dependent. For {\it Retailer}, most of the time is spent
on constructing $\bm G$ in Step 3, which is relatively large. For {\it
Favorita}, however, Step 2 takes the longest, since it contains one continuous
variable with many distinct values, and the DP algorithm for clustering runs in
time quadratic in the number of distinct values. The runtime for {\it Favorita}
could be improved by clustering this dimension with a different $k$-means
algorithm instead, but this may increase the approximation.

{\bf Comparison with mlpack.} The left columns of Table~\ref{table:rkmeans:comparison} compares the
runtime and approximation of Rk-means against mlpack on the three datasets for
different $k$ values with $\kappa = k$.  The approximation is given relative to the
objective value obtained by mlpack.  Speedup is given by comparing the
end-to-end performance of Rk-means and mlpack (ignoring disk I/O time), which for mlpack includes the
time needed by {\small \tt psql} to materialize $\bm X$.  Overall, Rk-means
often outperforms even just the clustering step from mlpack, and when
end-to-end computation is considered, Rk-means gives up to 115$\times$
speedup.  mlpack timed out after six hours for Favorita with $k = 50$.  In
addition, Rk-means has a much smaller memory footprint than mlpack: for
instance, on the {\it Favorita} dataset with $k = 20$, mlpack uses over 900GiB
of RAM to cluster the dataset, whereas Rk-means only requires 18Gib. In our
simulations, the approximation level is moderate, and consistently well below
the $9$-approximation bound from Theorem~\ref{thm:approx:bound}.

\input{e2etable}

{\bf Setting $\kappa < k$ for Step 2.}  We next evaluate the effect of setting
$\kappa$ to a smaller value than the number of clusters $k$.  This exploits the
speed/approximation tradeoff: smaller $\kappa$ helps reduce
the size of $\bm G$, at the cost of more approximation.
Table~\ref{table:rkmeans:comparison} presents for each dataset the results for
setting $k=20,\kappa = 10$ and $k=50, \kappa=20$, and compares them to the
relative performance and approximation over computing $k$ clusters in mlpack.

By setting $\kappa < k$, Rk-means can compute the $k$ clusters up to $208\times$
faster than mlpack and 3.6$\times$ faster than Rk-means with $\kappa = k$, while
the relative approximation remains moderate.  Our results are data
dependent---but as queries and databases scale, our speedups
will be even more significant.


\section{Conclusion}
\label{sec:conclusion}

We introduce Rk-means, a method to construct $k$-means clustering coresets on
relational data directly from the database.  Rk-means
gives a provably good clustering of the entire
dataset, without ever materializing the data set; this also yields
asymptotic improvements in running time.  Experimentally, we observe that the
coreset has size up to 180x smaller than the size of the data matrix and
this compression results in orders-of-magnitude improvements in the running
time, while still providing empirically good clusterings.  Although our work
here primarily focuses on $k$-means clustering, we believe our construction of 
grid coresets and the accompanied theory may be useful for other unsupervised 
learning tasks and plan to explore such possibilities in future work.


\end{document}

%% file: macros.tex
\usepackage[T1]{fontenc}    
\usepackage{hyperref}       
\usepackage{url}            
\usepackage{booktabs}       
\usepackage{amsfonts}       
\usepackage{nicefrac}       
\usepackage{microtype}      
\usepackage{wrapfig}

\usepackage{microtype}
\usepackage{graphicx}
\usepackage{subfigure}
\usepackage{booktabs} 
\usepackage{algorithm}
\usepackage{algorithmic}
\usepackage{fullpage}
\usepackage[english]{babel}
\usepackage[utf8x]{inputenc}
\usepackage[T1]{fontenc}

\usepackage{amsmath}
\usepackage{bm} 
\usepackage{sansmath}
\usepackage{amsfonts}
\usepackage{amsthm}
\usepackage{mathrsfs}

\usepackage{graphicx}
\usepackage{tikz}
\usetikzlibrary{calc}
\usetikzlibrary{shapes}

\usepackage[colorinlistoftodos]{todonotes}

\long\def\comment#1{}

\newcommand{\wgrid}{w_{\text{\sf grid}}}
\newcommand{\Pn}{P^{\text{in}}}
\newcommand{\Wasserstein}{\mathsf W_2}
\newcommand{\Wassersteinp}{\mathsf W_p}
\newcommand{\nop}[1]{}

\newcommand{\mycode}[1]{{\color{blue} \begin{verbatim} #1 \end{verbatim}}}
\newcommand{\suchthat}{\ | \ }

\DeclareMathOperator*{\argmin}{arg\,min}

\newcommand{\faqw}{\text{\small \sf faqw}}
\newcommand{\fhtw}{\text{\small \sf fhtw}}
\newcommand{\wkmeans}{\text{\small \sf wkmeans}}
\newcommand{\Xset}{\bm X}

\newcommand{\kmeans}{\text{\small \sf $k$-means}}

\newcommand{\city}{\text{\sf city}}

\newcommand{\opt}{\text{\small \sf OPT}}
\newcommand{\R}{{\mathbb R}}    			

\newcommand{\dom}{\text{\small \sf Dom}}

\newcommand{\norm}[1]{\left\|#1\right\|}                                        

\newcommand{\inner}[1]{\left\langle #1 \right\rangle}

\newcommand{\faq}{\text{\small \sf FAQ}}
\newcommand{\faqs}{\text{\small \sf FAQs}}
\newcommand{\InsideOut}{\text{\small \sf InsideOut}}

\newcommand{\calE}{\mathcal E}

\newcommand{\calF}{\mathcal F}

\theoremstyle{plain}                   
\newtheorem{thm}{Theorem}[section]
\newtheorem{lmm}[thm]{Lemma}
\newtheorem{prop}[thm]{Proposition}
\newtheorem{cor}[thm]{Corollary}

\theoremstyle{definition}              

\newtheorem{opm}{Question}
\newtheorem{conj}[thm]{Conjecture}

\newtheorem{defn}[thm]{Definition}

\newtheorem{rmk}[thm]{Remark}

\newcommand{\bi}{\begin{itemize}}
\newcommand{\ei}{\end{itemize}}
\newcommand{\be}{\begin{enumerate}}
\newcommand{\ee}{\end{enumerate}}

\newcommand{\bp}{\begin{proof}}                                                 
\newcommand{\ep}{\end{proof}}                                                   
\newcommand{\bcor}{\begin{cor}}                                                 
\newcommand{\ecor}{\end{cor}}                                                   
\newcommand{\bthm}{\begin{thm}}                                                 
\newcommand{\ethm}{\end{thm}}                                                   
\newcommand{\blmm}{\begin{lmm}}                                                 
\newcommand{\elmm}{\end{lmm}}                                                   
\newcommand{\bdefn}{\begin{defn}}                                               
\newcommand{\edefn}{\end{defn}}                                                 
\newcommand{\bprop}{\begin{prop}}                                               
\newcommand{\eprop}{\end{prop}}                                                 
\newcommand{\bconj}{\begin{conj}}                                               
\newcommand{\econj}{\end{conj}}                                                 
\newcommand{\bopm}{\begin{opm}}                                                 
\newcommand{\eopm}{\end{opm}}                                                   
\newcommand{\brmk}{\begin{rmk}}                                                 
\newcommand{\ermk}{\end{rmk}}

%% file: regular-workflow.tex
\filldraw [color=lightgray!20!white] (0, 0) rectangle (2.3, 3.0);
\draw (0, 0) rectangle (2.3, 3.0);

\filldraw [color=white] (0.3, 0.2) rectangle (0.8, 1.3);
\filldraw [color=white] (0.2, 1.7) rectangle (0.7, 2.8);
\filldraw [color=white] (0.9, 1.4) rectangle (1.4, 2.5);
\filldraw [color=white] (1.6, 1.6) rectangle (2.1, 2.7);
\filldraw [color=white] (1.6, 0.3) rectangle (2.1, 1.4);

\filldraw [color=lightgray] (0.3, 1.2) rectangle (0.8, 1.3);
\filldraw [color=lightgray] (0.3, 0.2) rectangle (0.55, 1.3);

\filldraw [color=lightgray] (0.2, 2.7) rectangle (0.7, 2.8);
\filldraw [color=lightgray] (0.2, 1.7) rectangle (0.45, 2.8);

\filldraw [color=lightgray] (0.9, 2.4) rectangle (1.4, 2.5);
\filldraw [color=lightgray] (0.9, 1.4) rectangle (1.15, 2.5);

\filldraw [color=lightgray] (1.6, 2.6) rectangle (2.1, 2.7);
\filldraw [color=lightgray] (1.6, 1.6) rectangle (1.85, 2.7);

\filldraw [color=lightgray] (1.6, 1.3) rectangle (2.1, 1.4);
\filldraw [color=lightgray] (1.6, 0.3) rectangle (1.85, 1.4);

\draw (0.3, 0.2) rectangle (0.8, 1.3);
\draw (0.2, 1.7) rectangle (0.7, 2.8);
\draw (0.9, 1.4) rectangle (1.4, 2.5);
\draw (1.6, 1.6) rectangle (2.1, 2.7);
\draw (1.6, 0.3) rectangle (2.1, 1.4);

\draw (0.3, 1.2) -- (0.8, 1.2);
\draw (0.55, 0.2) -- (0.55, 1.3);
\draw (0.3, 0.3) -- (0.8, 0.3);
\draw (0.3, 0.4) -- (0.8, 0.4);
\draw (0.3, 0.5) -- (0.8, 0.5);
\draw (0.3, 0.6) -- (0.8, 0.6);
\draw (0.3, 0.7) -- (0.8, 0.7);
\draw (0.3, 0.8) -- (0.8, 0.8);
\draw (0.3, 0.9) -- (0.8, 0.9);
\draw (0.3, 1.0) -- (0.8, 1.0);
\draw (0.3, 1.1) -- (0.8, 1.1);

\draw (0.2, 2.7) -- (0.7, 2.7);
\draw (0.45, 1.7) -- (0.45, 2.8);
\draw (0.2, 2.6) -- (0.7, 2.6);
\draw (0.2, 2.5) -- (0.7, 2.5);
\draw (0.2, 2.4) -- (0.7, 2.4);
\draw (0.2, 2.3) -- (0.7, 2.3);
\draw (0.2, 2.2) -- (0.7, 2.2);
\draw (0.2, 2.1) -- (0.7, 2.1);
\draw (0.2, 2.0) -- (0.7, 2.0);
\draw (0.2, 1.9) -- (0.7, 1.9);
\draw (0.2, 1.8) -- (0.7, 1.8);

\draw (0.9, 2.4) -- (1.4, 2.4);
\draw (1.15, 1.4) -- (1.15, 2.5);
\draw (0.9, 2.3) -- (1.4, 2.3);
\draw (0.9, 2.2) -- (1.4, 2.2);
\draw (0.9, 2.1) -- (1.4, 2.1);
\draw (0.9, 2.0) -- (1.4, 2.0);
\draw (0.9, 1.9) -- (1.4, 1.9);
\draw (0.9, 1.8) -- (1.4, 1.8);
\draw (0.9, 1.7) -- (1.4, 1.7);
\draw (0.9, 1.6) -- (1.4, 1.6);
\draw (0.9, 1.5) -- (1.4, 1.5);

\draw (1.6, 2.6) -- (2.1, 2.6);
\draw (1.85, 1.6) -- (1.85, 2.7);
\draw (1.6, 2.5) -- (2.1, 2.5);
\draw (1.6, 2.4) -- (2.1, 2.4);
\draw (1.6, 2.3) -- (2.1, 2.3);
\draw (1.6, 2.2) -- (2.1, 2.2);
\draw (1.6, 2.1) -- (2.1, 2.1);
\draw (1.6, 2.0) -- (2.1, 2.0);
\draw (1.6, 1.9) -- (2.1, 1.9);
\draw (1.6, 1.8) -- (2.1, 1.8);
\draw (1.6, 1.7) -- (2.1, 1.7);

\draw (1.6, 1.3) -- (2.1, 1.3);
\draw (1.85, 0.3) -- (1.85, 1.4);
\draw (1.6, 1.2) -- (2.1, 1.2);
\draw (1.6, 1.1) -- (2.1, 1.1);
\draw (1.6, 1.0) -- (2.1, 1.0);
\draw (1.6, 0.9) -- (2.1, 0.9);
\draw (1.6, 0.8) -- (2.1, 0.8);
\draw (1.6, 0.7) -- (2.1, 0.7);
\draw (1.6, 0.6) -- (2.1, 0.6);
\draw (1.6, 0.5) -- (2.1, 0.5);
\draw (1.6, 0.4) -- (2.1, 0.4);


\node[cylinder,
      cylinder uses custom fill,
      cylinder body fill=lightgray!50,
      cylinder end fill=lightgray!50,
      shape border rotate=90,
      aspect=0.85,
      minimum height=0.4cm,
      minimum width=0.9cm,
      draw] (disk) at (4.3, 1.0) {};
\node[cylinder,
      cylinder uses custom fill,
      cylinder body fill=lightgray!50,
      cylinder end fill=lightgray!50,
      shape border rotate=90,
      aspect=0.85,
      minimum height=0.4cm,
      minimum width=0.9cm,
      draw] (disk2) at (4.3, 1.45) {};
\node[cylinder,
      cylinder uses custom fill,
      cylinder body fill=lightgray!50,
      cylinder end fill=lightgray!50,
      shape border rotate=90,
      aspect=0.85,
      minimum height=0.25cm,
      minimum width=0.9cm,
      draw] (disk3) at (4.3, 1.9) {};



\filldraw [color=lightgray] (6.7, 3.1) rectangle (9.3, 3.3);
\filldraw [color=lightgray] (6.7, -0.3) rectangle (7.1, 3.3);
\draw (6.7, -0.3) rectangle (9.3, 3.3);
\draw (6.7, 3.1) -- (9.3, 3.1);
\draw (7.1, -0.3) -- (7.1, 3.3);
\draw (6.7, 2.9) -- (9.3, 2.9);
\draw (6.7, 2.7) -- (9.3, 2.7);
\draw (6.7, 2.5) -- (9.3, 2.5);
\draw (6.7, 2.3) -- (9.3, 2.3);
\draw (6.7, 2.1) -- (9.3, 2.1);
\draw (6.7, 1.9) -- (9.3, 1.9);
\draw (6.7, 1.7) -- (9.3, 1.7);
\draw (6.7, 1.5) -- (9.3, 1.5);
\draw (6.7, 1.3) -- (9.3, 1.3);
\draw (6.7, 1.1) -- (9.3, 1.1);
\draw (6.7, 0.9) -- (9.3, 0.9);
\draw (6.7, 0.7) -- (9.3, 0.7);
\draw (6.7, 0.5) -- (9.3, 0.5);
\draw (6.7, 0.3) -- (9.3, 0.3);
\draw (6.7, 0.1) -- (9.3, 0.1);
\draw (6.7, -0.1) -- (9.3, -0.1);

\draw (7.3, -0.3) -- (7.3, 3.3);
\draw (7.5, -0.3) -- (7.5, 3.3);
\draw (7.7, -0.3) -- (7.7, 3.3);
\draw (7.9, -0.3) -- (7.9, 3.3);
\draw (8.1, -0.3) -- (8.1, 3.3);
\draw (8.3, -0.3) -- (8.3, 3.3);
\draw (8.5, -0.3) -- (8.5, 3.3);
\draw (8.7, -0.3) -- (8.7, 3.3);
\draw (8.9, -0.3) -- (8.9, 3.3);
\draw (9.1, -0.3) -- (9.1, 3.3);


\filldraw [color=lightgray] (11.0, 2.0) rectangle (13.6, 2.2);
\filldraw [color=lightgray] (11.0, 1.2) rectangle (11.4, 2.2);

\draw (11.0, 1.2) rectangle (13.6, 2.2);

\draw (11.4, 1.2) -- (11.4, 2.2);
\draw (11.0, 2.0) -- (13.6, 2.0);

\draw (11.0, 1.8) -- (13.6, 1.8);
\draw (11.0, 1.6) -- (13.6, 1.6);
\draw (11.0, 1.4) -- (13.6, 1.4);

\draw (11.6, 1.2) -- (11.6, 2.2);
\draw (11.8, 1.2) -- (11.8, 2.2);
\draw (12.0, 1.2) -- (12.0, 2.2);
\draw (12.2, 1.2) -- (12.2, 2.2);
\draw (12.4, 1.2) -- (12.4, 2.2);
\draw (12.6, 1.2) -- (12.6, 2.2);
\draw (12.8, 1.2) -- (12.8, 2.2);
\draw (13.0, 1.2) -- (13.0, 2.2);
\draw (13.2, 1.2) -- (13.2, 2.2);
\draw (13.4, 1.2) -- (13.4, 2.2);


\node at (1.15, 4.0) {\scriptsize{RDBMS}};
\node at (4.3, 4.0) {\scriptsize{HDD}};
\node at (8.0, 4.0) {\scriptsize{dense in-memory}};
\node at (8.0, 3.65) {\scriptsize{data matrix}};
\node at (12.1, 4.0) {\scriptsize{centroids}};
\node at (12.1, 3.6) {\scriptsize{\it (results)}};


\draw (2.5, 1.9) -- (3.2, 1.9);
\draw (3.2, 1.9) -- (3.2, 2.1);
\draw (3.2, 2.1) -- (3.4, 1.65);
\draw (3.4, 1.65) -- (3.2, 1.2);
\draw (3.2, 1.2) -- (3.2, 1.4);
\draw (2.5, 1.4) -- (3.2, 1.4);
\draw (2.5, 1.4) -- (2.5, 1.9);
\draw (2.7, 1.55) -- (2.7, 1.75);
\draw (2.7, 1.55) -- (3.0, 1.75);
\draw (2.7, 1.75) -- (3.0, 1.55);
\draw (3.0, 1.55) -- (3.0, 1.75);

\draw (5.4, 1.9) -- (6.2, 1.9);
\draw (6.2, 1.9) -- (6.2, 2.1);
\draw (6.2, 2.1) -- (6.4, 1.65);
\draw (6.4, 1.65) -- (6.2, 1.2);
\draw (6.2, 1.2) -- (6.2, 1.4);
\draw (6.2, 1.4) -- (5.4, 1.4);
\draw (5.4, 1.4) -- (5.4, 1.9);

\draw (9.6, 1.9) -- (10.4, 1.9);
\draw (10.4, 1.9) -- (10.4, 2.1);
\draw (10.4, 2.1) -- (10.6, 1.65);
\draw (10.6, 1.65) -- (10.4, 1.2);
\draw (10.4, 1.2) -- (10.4, 1.4);
\draw (10.4, 1.4) -- (9.6, 1.4);
\draw (9.6, 1.4) -- (9.6, 1.9);

\node at (2.9, 0.7) {\scriptsize{\bf join}};
\node at (3.05, 0.25) {\scriptsize{\textbf{\it (FEQ)}}};

\node at (5.9, 0.75) {\scriptsize{\bf load}};

\node at (10.4, 0.75) {\scriptsize{\bf $k$-means}};
\node at (10.6, 0.25) {\scriptsize{\bf clustering}};

\node at (2.9, 2.4) {\scriptsize{\it (1)}};
\node at (5.9, 2.4) {\scriptsize{\it (2)}};
\node at (10.05, 2.4) {\scriptsize{\it (3)}};

%% file: rml-workflow.tex
\filldraw [color=lightgray!20!white] (0, 0) rectangle (2.3, 3.0);
\draw (0, 0) rectangle (2.3, 3.0);

\filldraw [color=white] (0.3, 0.2) rectangle (0.8, 1.3);
\filldraw [color=white] (0.2, 1.7) rectangle (0.7, 2.8);
\filldraw [color=white] (0.9, 1.4) rectangle (1.4, 2.5);
\filldraw [color=white] (1.6, 1.6) rectangle (2.1, 2.7);
\filldraw [color=white] (1.6, 0.3) rectangle (2.1, 1.4);

\filldraw [color=lightgray] (0.3, 1.2) rectangle (0.8, 1.3);
\filldraw [color=lightgray] (0.3, 0.2) rectangle (0.55, 1.3);

\filldraw [color=lightgray] (0.2, 2.7) rectangle (0.7, 2.8);
\filldraw [color=lightgray] (0.2, 1.7) rectangle (0.45, 2.8);

\filldraw [color=lightgray] (0.9, 2.4) rectangle (1.4, 2.5);
\filldraw [color=lightgray] (0.9, 1.4) rectangle (1.15, 2.5);

\filldraw [color=lightgray] (1.6, 2.6) rectangle (2.1, 2.7);
\filldraw [color=lightgray] (1.6, 1.6) rectangle (1.85, 2.7);

\filldraw [color=lightgray] (1.6, 1.3) rectangle (2.1, 1.4);
\filldraw [color=lightgray] (1.6, 0.3) rectangle (1.85, 1.4);

\draw (0.3, 0.2) rectangle (0.8, 1.3);
\draw (0.2, 1.7) rectangle (0.7, 2.8);
\draw (0.9, 1.4) rectangle (1.4, 2.5);
\draw (1.6, 1.6) rectangle (2.1, 2.7);
\draw (1.6, 0.3) rectangle (2.1, 1.4);

\draw (0.3, 1.2) -- (0.8, 1.2);
\draw (0.55, 0.2) -- (0.55, 1.3);
\draw (0.3, 0.3) -- (0.8, 0.3);
\draw (0.3, 0.4) -- (0.8, 0.4);
\draw (0.3, 0.5) -- (0.8, 0.5);
\draw (0.3, 0.6) -- (0.8, 0.6);
\draw (0.3, 0.7) -- (0.8, 0.7);
\draw (0.3, 0.8) -- (0.8, 0.8);
\draw (0.3, 0.9) -- (0.8, 0.9);
\draw (0.3, 1.0) -- (0.8, 1.0);
\draw (0.3, 1.1) -- (0.8, 1.1);

\draw (0.2, 2.7) -- (0.7, 2.7);
\draw (0.45, 1.7) -- (0.45, 2.8);
\draw (0.2, 2.6) -- (0.7, 2.6);
\draw (0.2, 2.5) -- (0.7, 2.5);
\draw (0.2, 2.4) -- (0.7, 2.4);
\draw (0.2, 2.3) -- (0.7, 2.3);
\draw (0.2, 2.2) -- (0.7, 2.2);
\draw (0.2, 2.1) -- (0.7, 2.1);
\draw (0.2, 2.0) -- (0.7, 2.0);
\draw (0.2, 1.9) -- (0.7, 1.9);
\draw (0.2, 1.8) -- (0.7, 1.8);

\draw (0.9, 2.4) -- (1.4, 2.4);
\draw (1.15, 1.4) -- (1.15, 2.5);
\draw (0.9, 2.3) -- (1.4, 2.3);
\draw (0.9, 2.2) -- (1.4, 2.2);
\draw (0.9, 2.1) -- (1.4, 2.1);
\draw (0.9, 2.0) -- (1.4, 2.0);
\draw (0.9, 1.9) -- (1.4, 1.9);
\draw (0.9, 1.8) -- (1.4, 1.8);
\draw (0.9, 1.7) -- (1.4, 1.7);
\draw (0.9, 1.6) -- (1.4, 1.6);
\draw (0.9, 1.5) -- (1.4, 1.5);

\draw (1.6, 2.6) -- (2.1, 2.6);
\draw (1.85, 1.6) -- (1.85, 2.7);
\draw (1.6, 2.5) -- (2.1, 2.5);
\draw (1.6, 2.4) -- (2.1, 2.4);
\draw (1.6, 2.3) -- (2.1, 2.3);
\draw (1.6, 2.2) -- (2.1, 2.2);
\draw (1.6, 2.1) -- (2.1, 2.1);
\draw (1.6, 2.0) -- (2.1, 2.0);
\draw (1.6, 1.9) -- (2.1, 1.9);
\draw (1.6, 1.8) -- (2.1, 1.8);
\draw (1.6, 1.7) -- (2.1, 1.7);

\draw (1.6, 1.3) -- (2.1, 1.3);
\draw (1.85, 0.3) -- (1.85, 1.4);
\draw (1.6, 1.2) -- (2.1, 1.2);
\draw (1.6, 1.1) -- (2.1, 1.1);
\draw (1.6, 1.0) -- (2.1, 1.0);
\draw (1.6, 0.9) -- (2.1, 0.9);
\draw (1.6, 0.8) -- (2.1, 0.8);
\draw (1.6, 0.7) -- (2.1, 0.7);
\draw (1.6, 0.6) -- (2.1, 0.6);
\draw (1.6, 0.5) -- (2.1, 0.5);
\draw (1.6, 0.4) -- (2.1, 0.4);


\filldraw [color=white] (6.3, 0.2) rectangle (6.8, 1.3);
\filldraw [color=white] (6.2, 1.7) rectangle (6.7, 2.8);
\filldraw [color=white] (6.9, 1.4) rectangle (7.4, 2.5);
\filldraw [color=white] (7.6, 1.6) rectangle (8.1, 2.7);
\filldraw [color=white] (7.6, 0.3) rectangle (8.1, 1.4);

\filldraw [color=lightgray] (6.3, 1.2) rectangle (6.8, 1.3);
\filldraw [color=lightgray] (6.3, 0.2) rectangle (6.55, 1.3);

\filldraw [color=lightgray] (6.2, 2.7) rectangle (6.7, 2.8);
\filldraw [color=lightgray] (6.2, 1.7) rectangle (6.45, 2.8);

\filldraw [color=lightgray] (6.9, 2.4) rectangle (7.4, 2.5);
\filldraw [color=lightgray] (6.9, 1.4) rectangle (7.15, 2.5);

\filldraw [color=lightgray] (7.6, 2.6) rectangle (8.1, 2.7);
\filldraw [color=lightgray] (7.6, 1.6) rectangle (7.85, 2.7);

\filldraw [color=lightgray] (7.6, 1.3) rectangle (8.1, 1.4);
\filldraw [color=lightgray] (7.6, 0.3) rectangle (7.85, 1.4);

\draw (6.3, 0.2) rectangle (6.8, 1.3);
\draw (6.2, 1.7) rectangle (6.7, 2.8);
\draw (6.9, 1.4) rectangle (7.4, 2.5);
\draw (7.6, 1.6) rectangle (8.1, 2.7);
\draw (7.6, 0.3) rectangle (8.1, 1.4);

\draw (6.3, 1.2) -- (6.8, 1.2);
\draw (6.55, 0.2) -- (6.55, 1.3);
\draw (6.3, 0.3) -- (6.8, 0.3);
\draw (6.3, 0.4) -- (6.8, 0.4);
\draw (6.3, 0.5) -- (6.8, 0.5);
\draw (6.3, 0.6) -- (6.8, 0.6);
\draw (6.3, 0.7) -- (6.8, 0.7);
\draw (6.3, 0.8) -- (6.8, 0.8);
\draw (6.3, 0.9) -- (6.8, 0.9);
\draw (6.3, 1.0) -- (6.8, 1.0);
\draw (6.3, 1.1) -- (6.8, 1.1);

\draw (6.2, 2.7) -- (6.7, 2.7);
\draw (6.45, 1.7) -- (6.45, 2.8);
\draw (6.2, 2.6) -- (6.7, 2.6);
\draw (6.2, 2.5) -- (6.7, 2.5);
\draw (6.2, 2.4) -- (6.7, 2.4);
\draw (6.2, 2.3) -- (6.7, 2.3);
\draw (6.2, 2.2) -- (6.7, 2.2);
\draw (6.2, 2.1) -- (6.7, 2.1);
\draw (6.2, 2.0) -- (6.7, 2.0);
\draw (6.2, 1.9) -- (6.7, 1.9);
\draw (6.2, 1.8) -- (6.7, 1.8);

\draw (6.9, 2.4) -- (7.4, 2.4);
\draw (7.15, 1.4) -- (7.15, 2.5);
\draw (6.9, 2.3) -- (7.4, 2.3);
\draw (6.9, 2.2) -- (7.4, 2.2);
\draw (6.9, 2.1) -- (7.4, 2.1);
\draw (6.9, 2.0) -- (7.4, 2.0);
\draw (6.9, 1.9) -- (7.4, 1.9);
\draw (6.9, 1.8) -- (7.4, 1.8);
\draw (6.9, 1.7) -- (7.4, 1.7);
\draw (6.9, 1.6) -- (7.4, 1.6);
\draw (6.9, 1.5) -- (7.4, 1.5);

\draw (7.6, 2.6) -- (8.1, 2.6);
\draw (7.85, 1.6) -- (7.85, 2.7);
\draw (7.6, 2.5) -- (8.1, 2.5);
\draw (7.6, 2.4) -- (8.1, 2.4);
\draw (7.6, 2.3) -- (8.1, 2.3);
\draw (7.6, 2.2) -- (8.1, 2.2);
\draw (7.6, 2.1) -- (8.1, 2.1);
\draw (7.6, 2.0) -- (8.1, 2.0);
\draw (7.6, 1.9) -- (8.1, 1.9);
\draw (7.6, 1.8) -- (8.1, 1.8);
\draw (7.6, 1.7) -- (8.1, 1.7);

\draw (7.6, 1.3) -- (8.1, 1.3);
\draw (7.85, 0.3) -- (7.85, 1.4);
\draw (7.6, 1.2) -- (8.1, 1.2);
\draw (7.6, 1.1) -- (8.1, 1.1);
\draw (7.6, 1.0) -- (8.1, 1.0);
\draw (7.6, 0.9) -- (8.1, 0.9);
\draw (7.6, 0.8) -- (8.1, 0.8);
\draw (7.6, 0.7) -- (8.1, 0.7);
\draw (7.6, 0.6) -- (8.1, 0.6);
\draw (7.6, 0.5) -- (8.1, 0.5);
\draw (7.6, 0.4) -- (8.1, 0.4);


\filldraw [color=lightgray] (11.0, 2.0) rectangle (13.6, 2.2);
\filldraw [color=lightgray] (11.0, 1.2) rectangle (11.4, 2.2);

\draw (11.0, 1.2) rectangle (13.6, 2.2);

\draw (11.4, 1.2) -- (11.4, 2.2);
\draw (11.0, 2.0) -- (13.6, 2.0);

\draw (11.0, 1.8) -- (13.6, 1.8);
\draw (11.0, 1.6) -- (13.6, 1.6);
\draw (11.0, 1.4) -- (13.6, 1.4);

\draw (11.6, 1.2) -- (11.6, 2.2);
\draw (11.8, 1.2) -- (11.8, 2.2);
\draw (12.0, 1.2) -- (12.0, 2.2);
\draw (12.2, 1.2) -- (12.2, 2.2);
\draw (12.4, 1.2) -- (12.4, 2.2);
\draw (12.6, 1.2) -- (12.6, 2.2);
\draw (12.8, 1.2) -- (12.8, 2.2);
\draw (13.0, 1.2) -- (13.0, 2.2);
\draw (13.2, 1.2) -- (13.2, 2.2);
\draw (13.4, 1.2) -- (13.4, 2.2);


\node at (1.15, 3.5) {\scriptsize{RDBMS}};
\node at (7.2, 3.5) {\scriptsize{clustered}};
\node at (7.2, 3.15) {\scriptsize{relations}};
\node at (12.1, 3.5) {\scriptsize{centroids}};
\node at (12.1, 3.1) {\scriptsize{\it (results)}};


\draw (2.9, 1.9) -- (5.2, 1.9);
\draw (5.2, 1.9) -- (5.2, 2.1);
\draw (5.2, 2.1) -- (5.4, 1.65);
\draw (5.4, 1.65) -- (5.2, 1.2);
\draw (5.2, 1.2) -- (5.2, 1.4);
\draw (2.9, 1.4) -- (5.2, 1.4);
\draw (2.9, 1.4) -- (2.9, 1.9);

\draw (8.6, 1.9) -- (10.4, 1.9);
\draw (10.4, 1.9) -- (10.4, 2.1);
\draw (10.4, 2.1) -- (10.6, 1.65);
\draw (10.6, 1.65) -- (10.4, 1.2);
\draw (10.4, 1.2) -- (10.4, 1.4);
\draw (10.4, 1.4) -- (8.6, 1.4);
\draw (8.6, 1.4) -- (8.6, 1.9);

\node at (9.5, 1.63) {\scriptsize{\faq}};

\node at (3.5, 0.9) {\scriptsize{\textbf{$k$-means}}};
\node at (4.08, 0.51) {\scriptsize{\textbf{on individual}}};
\node at (4.03, 0.1) {\scriptsize{\textbf{DB relations}}};

\node at (9.5, 0.86) {\scriptsize{\textbf{weighted}}};
\node at (9.75, 0.51) {\scriptsize{\textbf{$k$-means on}}};
\node at (10.0, 0.1) {\scriptsize{\textbf{cross product}}};

%% file: rkmeans_breakdown.tex
\begingroup
  \makeatletter
  \providecommand\color[2][]{%
    \GenericError{(gnuplot) \space\space\space\@spaces}{%
      Package color not loaded in conjunction with
      terminal option `colourtext'%
    }{See the gnuplot documentation for explanation.%
    }{Either use 'blacktext' in gnuplot or load the package
      color.sty in LaTeX.}%
    \renewcommand\color[2][]{}%
  }%
  \providecommand\includegraphics[2][]{%
    \GenericError{(gnuplot) \space\space\space\@spaces}{%
      Package graphicx or graphics not loaded%
    }{See the gnuplot documentation for explanation.%
    }{The gnuplot epslatex terminal needs graphicx.sty or graphics.sty.}%
    \renewcommand\includegraphics[2][]{}%
  }%
  \providecommand\rotatebox[2]{#2}%
  \@ifundefined{ifGPcolor}{%
    \newif\ifGPcolor
    \GPcolortrue
  }{}%
  \@ifundefined{ifGPblacktext}{%
    \newif\ifGPblacktext
    \GPblacktextfalse
  }{}%
  \let\gplgaddtomacro\g@addto@macro
  \gdef\gplbacktext{}%
  \gdef\gplfronttext{}%
  \makeatother
  \ifGPblacktext
    \def\colorrgb#1{}%
    \def\colorgray#1{}%
  \else
    \ifGPcolor
      \def\colorrgb#1{\color[rgb]{#1}}%
      \def\colorgray#1{\color[gray]{#1}}%
      \expandafter\def\csname LTw\endcsname{\color{white}}%
      \expandafter\def\csname LTb\endcsname{\color{black}}%
      \expandafter\def\csname LTa\endcsname{\color{black}}%
      \expandafter\def\csname LT0\endcsname{\color[rgb]{1,0,0}}%
      \expandafter\def\csname LT1\endcsname{\color[rgb]{0,1,0}}%
      \expandafter\def\csname LT2\endcsname{\color[rgb]{0,0,1}}%
      \expandafter\def\csname LT3\endcsname{\color[rgb]{1,0,1}}%
      \expandafter\def\csname LT4\endcsname{\color[rgb]{0,1,1}}%
      \expandafter\def\csname LT5\endcsname{\color[rgb]{1,1,0}}%
      \expandafter\def\csname LT6\endcsname{\color[rgb]{0,0,0}}%
      \expandafter\def\csname LT7\endcsname{\color[rgb]{1,0.3,0}}%
      \expandafter\def\csname LT8\endcsname{\color[rgb]{0.5,0.5,0.5}}%
    \else
      \def\colorrgb#1{\color{black}}%
      \def\colorgray#1{\color[gray]{#1}}%
      \expandafter\def\csname LTw\endcsname{\color{white}}%
      \expandafter\def\csname LTb\endcsname{\color{black}}%
      \expandafter\def\csname LTa\endcsname{\color{black}}%
      \expandafter\def\csname LT0\endcsname{\color{black}}%
      \expandafter\def\csname LT1\endcsname{\color{black}}%
      \expandafter\def\csname LT2\endcsname{\color{black}}%
      \expandafter\def\csname LT3\endcsname{\color{black}}%
      \expandafter\def\csname LT4\endcsname{\color{black}}%
      \expandafter\def\csname LT5\endcsname{\color{black}}%
      \expandafter\def\csname LT6\endcsname{\color{black}}%
      \expandafter\def\csname LT7\endcsname{\color{black}}%
      \expandafter\def\csname LT8\endcsname{\color{black}}%
    \fi
  \fi
    \setlength{\unitlength}{0.0500bp}%
    \ifx\gptboxheight\undefined%
      \newlength{\gptboxheight}%
      \newlength{\gptboxwidth}%
      \newsavebox{\gptboxtext}%
    \fi%
    \setlength{\fboxrule}{0.5pt}%
    \setlength{\fboxsep}{1pt}%
\begin{picture}(8220.00,2266.00)%
    \gplgaddtomacro\gplbacktext{%
      \csname LTb\endcsname
      \put(814,440){\makebox(0,0)[r]{\strut{}$0$}}%
      \csname LTb\endcsname
      \put(814,761){\makebox(0,0)[r]{\strut{}$100$}}%
      \csname LTb\endcsname
      \put(814,1082){\makebox(0,0)[r]{\strut{}$200$}}%
      \csname LTb\endcsname
      \put(814,1403){\makebox(0,0)[r]{\strut{}$300$}}%
      \csname LTb\endcsname
      \put(814,1724){\makebox(0,0)[r]{\strut{}$400$}}%
      \csname LTb\endcsname
      \put(814,2045){\makebox(0,0)[r]{\strut{}$500$}}%
      \put(1432,280){\makebox(0,0){\strut{}5}}%
      \put(1674,280){\makebox(0,0){\strut{}10}}%
      \put(1917,280){\makebox(0,0){\strut{}20}}%
      \put(2160,280){\makebox(0,0){\strut{}50}}%
      \put(2403,280){\makebox(0,0){\strut{}X}}%
      \put(3131,280){\makebox(0,0){\strut{}5}}%
      \put(3374,280){\makebox(0,0){\strut{}10}}%
      \put(3616,280){\makebox(0,0){\strut{}20}}%
      \put(3859,280){\makebox(0,0){\strut{}50}}%
      \put(4102,280){\makebox(0,0){\strut{}X}}%
      \put(4830,280){\makebox(0,0){\strut{}5}}%
      \put(5073,280){\makebox(0,0){\strut{}10}}%
      \put(5316,280){\makebox(0,0){\strut{}20}}%
      \put(5558,280){\makebox(0,0){\strut{}50}}%
      \put(5801,280){\makebox(0,0){\strut{}X}}%
    }%
    \gplgaddtomacro\gplfronttext{%
      \csname LTb\endcsname
      \put(330,1242){\rotatebox{-270}{\makebox(0,0){\strut{}Wall-clock time (sec)}}}%
      \csname LTb\endcsname
      \put(8087,1962){\makebox(0,0)[r]{\strut{}Step 1}}%
      \csname LTb\endcsname
      \put(8087,1797){\makebox(0,0)[r]{\strut{}Step 2}}%
      \csname LTb\endcsname
      \put(8087,1632){\makebox(0,0)[r]{\strut{}Step 3}}%
      \csname LTb\endcsname
      \put(8087,1467){\makebox(0,0)[r]{\strut{}Step 4}}%
      \csname LTb\endcsname
      \put(8087,1302){\makebox(0,0)[r]{\strut{}Data Matrix}}%
      \csname LTb\endcsname
      \put(5194,99){\makebox(0,0){\strut{}Yelp}}%
      \put(3495,99){\makebox(0,0){\strut{}Favorita}}%
      \put(1796,99){\makebox(0,0){\strut{}Retailer}}%
    }%
    \gplbacktext
    \put(0,0){\includegraphics{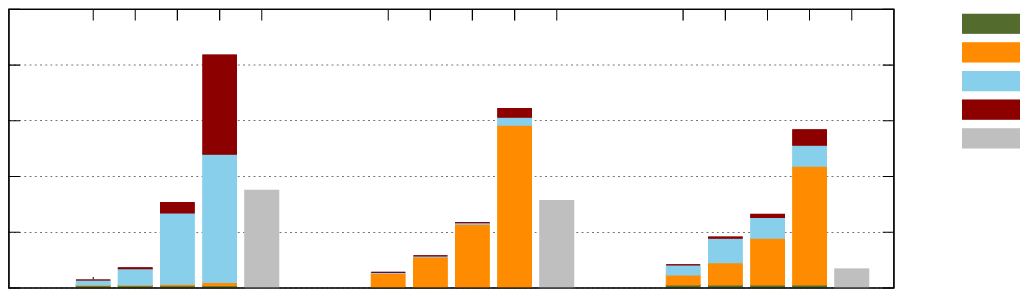}}%
    \gplfronttext
  \end{picture}%
\endgroup

%% file: e2etable.tex
\begin{table*}[t!]
\centering \small
\begin{tabular}{|l|r|r|r|r||r|r|}\hline
  \textbf{Retailer}      & k = 5         & k = 10         & k = 20         & k = 50           & k=20, $\kappa$ = 10 & k = 50, $\kappa$ = 20 \\\hline
  Compute $\bm X$ (psql) & 175.47        & 175.47         & 175.47         & 175.47           & 175.47              & 175.47                \\
  Clustering (mlpack)    & 65.41         & 158.81         & 385.67         & 1,453.88         & 385.67              & 1,453.88              \\\hline
  Rk-means               & 15.66         & 54.59          & 230.17         & 650.20           & 63.51               & 344.31                \\\hline
  Relative Speedup       & 15.38$\times$ & 6.12$\times$   & 2.44$\times$   & 2.51$\times$     & 8.84$\times$        & 4.73$\times$          \\
  Relative Approx.       & 0.20          & 0.08           & 0.03           & 0.00             & 0.03                & 0.02                  \\\hline\hline
  \textbf{Favorita}      & k = 5         & k = 10         & k = 20         & k = 50           & k=20, $\kappa$ = 10 & k = 50, $\kappa$ = 20 \\\hline
  Compute $\bm X$ (psql) & 156.86        & 156.86         & 156.86         & 156.86           & 156.86              & 156.86                \\
  Clustering (mlpack)    & 1,002.54      & 6,449.32       & 11,794.49      & \!\!$>$21,600.00 & 11,794.49           & $>$21,600             \\\hline
  Rk-means               & 27.95         & 57.72          & 118.36         & 334.65           & 57.65               & 120.77                \\\hline
  Relative Speedup       & 41.49$\times$ & 114.59$\times$ & 100.98$\times$ & $>$64.55$\times$ & 207.30$\times$      & $>$178.86$\times$     \\
  Relative Approx.       & 2.99          & 0.35           & 0.12           & --               & 1.93                & --                    \\\hline\hline
  \textbf{Yelp}          & k = 5         & k = 10         & k = 20         & k = 50           & k=20, $\kappa$ = 10 & k = 50, $\kappa$ = 20 \\\hline
  Compute $\bm X$ (psql) & 33.83         & 33.83          & 33.83          & 33.83            & 33.83               & 33.83                 \\
  Clustering (mlpack)    & 210.59        & 640.43         & 2,107.83       & 11,474.24        & 2,107.83            & 11,474.24        \\\hline
  Rk-means               & 43.37         & 107.71         & 195.22         & 405.11           & 114.34              & 241.34                \\\hline
  Relative Speedup       & 5.64$\times$  & 6.26$\times$   & 10.97$\times$  & 28.41$\times$    & 18.73$\times$       & 47.68$\times$         \\
  Relative Approx.       & 0.37          & 0.26           & 0.13           & 0.05             & 0.27                & 0.20                  \\\hline
\end{tabular}
\caption{End-to-end runtime and approximation comparison of Rk-means and mlpack
  on each dataset. The first four columns use different $\kappa = k$ values; the
  last two show results for setting $\kappa < k$.}
\label{table:rkmeans:comparison}
\end{table*}